\pdfoutput=1
\documentclass{article}
\usepackage[nottoc]{tocbibind}
\usepackage[nohead, nomarginpar, margin=1in, foot=.25in]{geometry}
\usepackage{mathtools}
\usepackage{bbm}
\usepackage{amssymb}
\usepackage{float}  

\usepackage{amsmath}
\usepackage{lscape} 
\usepackage{graphicx}
\graphicspath{ {./images/} }

\usepackage{graphicx,times,verbatim,amsmath,amssymb,amsthm}
\usepackage{bm}
\usepackage{enumerate}
\usepackage{mathrsfs}
\usepackage{epsfig}
\usepackage{caption}
\usepackage{subcaption}
\usepackage{verbatim}
\usepackage{multirow}
\usepackage{array}
\usepackage{tikz}

\newtheorem{theorem}{Theorem}[section]
\newtheorem{lemma}[theorem]{Lemma}

\newtheorem{definition}[theorem]{Definition}

\newtheorem{assumption}[theorem]{Assumption}

\newtheorem*{remark*}{Remark}
\DeclareGraphicsExtensions{.eps}

\usepackage[linesnumbered,ruled]{algorithm2e}
\RestyleAlgo{boxruled}

-\headsep \textheight 8.9in \oddsidemargin 0pt \evensidemargin
\oddsidemargin \marginparwidth 0.5in \textwidth 6.5in

\usepackage{titling}
\usepackage{blindtext}

\usepackage{hyperref}

\newcommand{\block}[1]{
  \underbrace{\begin{matrix}1 & \cdots & 1\end{matrix}}_{#1}
}

\title{Factor Analysis on Citations, Using a Combined Latent and Logistic Regression Model}

\usepackage[T1]{fontenc}
\usepackage{authblk}

\author{
    Namjoon Suh
    \thanks{Namjoon Suh is with School of Industrial and Systems Engineering, Georgia Institute of Technology, 755 Ferst Dr, Atlanta, GA, USA. E-mail : namjsuh@gatech.edu } \quad
     Xiaoming Huo 
    \thanks{Xiaoming Huo is with  with School of Industrial and Systems Engineering, Georgia Institute of Technology, 755 Ferst Dr, Atlanta, GA, USA. E-mail : huo@gatech.edu} \quad
     Eric Heim 
    \thanks{Eric Heim is with Software Engineering Institute, Carnegie Mellon University, Pittsburgh, PA, USA.} \quad
     Lee Seversky 
    \thanks{Lee Seversky is with Information System Division (AFRL/RIS), the Air Force Research Laboratory, Department of the Air Force, Air Force,  AFRL/RIK - Rome, 26 ElectronicParkway, Rome, NewYork, USA.}
}

\date{\vspace{-8ex}}

\begin{document}
\begin{titlingpage}
\maketitle

\begin{abstract}
We propose a combined model, which integrates the latent factor model and the logistic regression model, for the citation network.
It is noticed that neither a latent factor model nor a logistic regression model alone is sufficient to capture the structure of the data.
The proposed model has a latent (i.e., factor analysis) model to represents the main technological trends (a.k.a., factors), and adds a sparse component that captures the remaining ad-hoc dependence.
Parameter estimation is carried out through the construction of a joint-likelihood function of edges and properly chosen penalty terms.
The convexity of the objective function allows us to develop an efficient algorithm, while the penalty terms push towards a low-dimensional latent component and a sparse graphical structure.
Simulation results show that the proposed method works well in practical situations.
The proposed method has been applied to a real application, which contains a citation network of statisticians (Ji and Jin, 2016 \cite{ji2016coauthorship}).
Some interesting findings are reported.
\end{abstract}



\section{Introduction}
\label{sec:intro}
We study a citation network, where each node (i.e., item) can be a technical report or a publication.
A node may cite another node.
Associated with a pair of nodes $i$ and $j$, we denote a binary random variable $X_{ij}$, where $1 \le i,j \le n$ and $n$ is the total number of nodes.
We have $X_{ij}=1$ if and only if either node $i$ cites node $j$ or vice versa; otherwise $X_{ij}=0$.
For each node $i$, we assume that there is an associated binary vector $f_{i} \in \mathbb{R}^K$, such that the $k$th entry of $f_{i}$, $f_{ik} = 1$, if and only if node $i$ is related to topic (i.e., factor) $k, 1\le k \le K$.
Here $K$ is the total number of underlying topics (i.e., factors, or trends).
We assume a logistic model for $X_{ij}$'s: for $1\le i,j \le n$,

\begin{equation}
\label{eq:logistic01}
\mathbb{P}(X_{ij}=1) = \frac{e^{\alpha + f_i^T D f_j }}{1 + e^{\alpha + f_i^T D f_j }},
\end{equation}
where $\alpha \in \mathbb{R}$ is a parameter and matrix $D \in \mathbb{R}^{K \times K}$ is a diagonal matrix: $D = \mbox{diag}\{d_{1},d_{2},\ldots,d_{K}\}$.
We assume $d_{i} > 0$ for $1 \le i \le K$.
Another way to put \eqref{eq:logistic01} is
\begin{equation}
\label{eq:logistic02}
\mathbb{P}(X_{ij}=1) = \frac{\exp\left(\alpha + \sum_{k=1}^K f_{ik}f_{jk}d_{k} \right)}{1 + \exp\left(\alpha + \sum_{k=1}^K f_{ik}f_{jk}d_{k} \right)}.
\end{equation}
A justification of the above model is that when both node $i$ and node $j$ are related to topic $k$, they have a higher chance to cite one way or the other.
We have assumed a common strength coefficient $d_k$ ($1\le k \le K$) for factor $k$, despite different nodes.
We denote a matrix $F = \{f_1, f_2, \ldots, f_n\} \in \mathbb{R}^{K \times n}$.
Each column $i$ in matrix $F$ contains the factor loadings associated with the node $i$ ($1\le i \le n$).
Given the diagonal matrix $D$ and the factor loading matrix $F$, we assume that $X_{ij}$'s are independent; therefore we have the total conditional probability function as follows:
\begin{equation}
\label{eq:logistic03}
\mathbb{P}(\{X_{ij}, 1\le i,j \le n\})
= \prod_{1\le i<j \le n} \mathbb{P}(X_{ij})
= \prod_{1\le i<j \le n}  \frac{e^{X_{ij}(\alpha + f_i^T D f_j) }}{1 + e^{\alpha + f_i^T D f_j }},
\end{equation}
where $\mathbb{P}(X_{ij})$ is given in \eqref{eq:logistic02}.
The last equation holds because $X_{ij}$ only takes binary (i.e., $0$ or $1$) values.
Recall that the dot product of two matrices with same dimensionality, $A,B\in \mathbb{R}^{a \times b}$, is defined as $A\bullet B=\mbox{trace}(A^T B) = \sum_{i=1}^a\sum_{j=1}^b a_{ij}b_{ij}$.
The above \eqref{eq:logistic03} can be further rewritten as
\begin{equation}
\label{eq:logistic04}
\mathbb{P}(\{X_{ij}, 1\le i,j \le n\})
= \frac{\exp\left(\alpha \sum_{1\le i< j\le n}X_{ij} +\frac{1}{2} X \bullet (F^T D F)\right)}{\prod_{1\le i<j \le n}  1 + e^{\alpha + f_i^T D f_j }},
\end{equation}
where we assume $X_{ii}=0$ for all $i$ ($1\le i \le n$) and $X_{ij} = X_{ji}$ for all $i$ and $j$ ($1\le i,j \le n$), i.e., the matrix $X$ is symmetric.
The above delivers a factor analysis model.
Various linear and nonlinear latent variable models have been studied extensively in the literature (e.g., \cite{joreskog1969general, mcdonald2014factor,lord2008statistical, rasch1980probabilistic, harman1960modern, joreskog1970general}).

Our work is motivated from a recent work named {\emph{Fused Latent and Graphical (FLaG) model} (Chen et al, 2016, \cite{chen2016fused}).
They assume that majority of variation of responses can be accounted by low dimensional latent vector, and remaining dependent structure of responses can be explained by sparse graphical structure.
Thus, the resulting model contains a low-dimensional latent vector and a sparse conditional graph.
Their key idea is to separate these two dependent structures so that they can facilitate the statistical inference.
In our model, we also assume that there exist two dependent structures among citation edges in a network.
A low-dimensional version of the aforementioned latent vector model is largely correct and majority of the citations among the nodes are induced by these common latent vectors $f_{i}$'s (with weight coefficients $d_{i}$'s).
There is still a small remainder due to the sparse graphical component.

Though it may seem similar to Chen et al \cite{chen2016fused}, we work on a different model formulation in several aspects.
We summarize the differences as follows.
\begin{enumerate}
\item FLaG is built to analyze the Eysenck's Personality Questionnaire that consists of items designed to measure Psychoticism, Extraversion, and Neuroticism.
    So there are $p$ questions that need to be answered, and each questions fall into above three categories.
    If there are $n$ respondents to questions, they have $n$ independent data generated from the same distribution.
    In our case, the observed citation network can be thought of as one realization of a random graph.

\item In FLaG model, a collection of binary responses for each question in the questionnaire follows a joint distribution, which is a combination of the Item Response Theory (IRT) model and the Ising model.
    We model the citation edges among papers as random variables, whose dependent structure is characterized by the combination of the Latent Factor Analysis model and the Sparse Graphical model.

\item FLaG approximates the original likelihood through constructing pseudo-likelihood function by taking advantage of conditional independence among the nodes.
    In our model, likelihood function is directly accessible due to the conditional independence among edges.
\end{enumerate}

The proposed modeling framework is also related with the analysis of decomposing a matrix into low-rank and sparse components (\cite{agarwal2012noisy,candes2011robust,chandrasekaran2010latent, zhou2010stable}).
Specifically, paper \cite{chandrasekaran2010latent} studies statistical inference of a multivariate Gaussian model whose precision matrix admits the form of a low-rank matrix plus a sparse matrix.
The inference and optimization of the current model are different from the aforementioned cases.
We will construct a regularized-likelihood function, based on which estimator will be proposed for simultaneous model selection and parameter estimation.
The objective function in the optimization problem for the regularized estimator is convex, for which we will develop an efficient algorithm through the alternating direction method of multiplier (ADMM, \cite{boyd2011distributed, gabay1975dual, glowinski1975solution}).

The rest of the paper is organized as follows.
In Section \ref{sec:model-form}, we will give a presentation on how to build a model, which can encode both the latent dependent structure due to the common topics and the remaining sparse ad-hoc dependent structure.
In Section \ref{sec:estimate}, we will discuss the assumptions in our model and the penalization on the likelihood function, which is constructed in Section \ref{sec:model-form}.
In Section \ref{sec:theorem}, we provide a non-asymptotic error bound of the estimator.
Section \ref{sec:compute} gives the detailed procedure on how to compute the estimator of the optimization problem.
In Section \ref{sec:Num-anal}, we will present simple numerical experiments on synthetic data, as well as an application of our model on a real citation network of statisticians.
We finally conclude this work in Section \ref{sec:Con} with several open questions and some possible directions of future research.

\section{Model Formulation}
\label{sec:model-form}

Recall the following graphical model that was established in \eqref{eq:logistic04}, which is essentially a factor model (latent variable model):
$$
\mathbb{P}\left(\{X_{ij}, 1\le i,j \le n\}\right)
= \frac{\exp\left(\alpha \sum_{1\le i< j\le n}X_{ij} +\frac{1}{2} X \bullet (F^T D F)\right)}{\prod_{1\le i<j \le n}  1 + e^{\alpha + f_i^T D f_j }},
$$
where
$X_{ij}, 1\le i,j \le n$, are binary random variables indicating either node $i$ cites node $j$ or vice versa,
matrix $X = \{X_{ij}\} \in \mathbb{R}^{n \times n}$ is symmetric with diagonal entries all being equal to zero,
factor loading matrix $F = [f_1,f_2,\ldots,f_n]\in \mathbb{R}^{K\times n}$ records the relation between nodes and the underlying topics, $F^T$ is the transpose of $F$,
and matrix $D \in \mathbb{R}^{K \times K}$  is diagonal with entries being the weight coefficients of factors.

The above specifies a latent model (or equivalently a factor model).
We now describe a graphical model as follows.
The graphical model will complement the latent model by characterizing links that are not interpretable via common factors.
For the aforementioned binary random variable $X_{ij}$, $1\le i,j \le n$, we define
\begin{equation}
\label{eq:logistic05}
\mathbb{P}(X_{ij}=1) = \frac{e^{\alpha' + S_{ij} }}{1 + e^{\alpha' + S_{ij} }},
\end{equation}
where $S_{ij} \in \mathbb{R}$, for $1\le i,j \le n$, denotes the relation between nodes $i$ and $j$.
Note that the matrix $S$ is introduced to capture the ad-hoc links in the graph.
If we have $S_{ij}\leq0$, then it is less likely to have a citational relationship between nodes $i$ and $j$.
On the other hand, if $S_{ij}>0$, then it is more likely to have a citation link between nodes $i$ and $j$.
Here parameter $\alpha' \in \mathbb{R}$ plays the same role as parameter $\alpha$ does in model \eqref{eq:logistic01}.
Denote the matrix $S = \{S_{ij}, 1\le i,j \le n \} \in \mathbb{R}^{n \times n}$.
Assume that given the matrix $S$, the binary random variables
$X_{ij}$'s are independent;
consequently, we have the total conditional probability function as follows:
\begin{eqnarray}
\mathbb{P}(\{X_{ij}, 1\le i,j \le n\})
&=& \prod_{1\le i<j \le n} \mathbb{P}(X_{ij}) \nonumber \\
&=& \prod_{1\le i<j \le n}  \frac{e^{X_{ij}(\alpha' + S_{ij}) }}{1 + e^{\alpha' + S_{ij} }} \nonumber \\
&=& \frac{\exp\left(\alpha' \sum_{1\le i< j\le n}X_{ij} +\frac{1}{2} X \bullet S\right)}{\prod_{1\le i<j \le n}  1 + e^{\alpha' + S_{ij} }}.
\label{eq:logistic06}
\end{eqnarray}
Recall that we have assumed that $X_{ii}=0$ for all $i$ ($1\le i \le n$) and $X_{ij} = X_{ji}$ for all $i$ and $j$ ($1\le i,j \le n$), i.e., the matrix $X$ is symmetric.
In the combined model, we integrate \eqref{eq:logistic04} and
\eqref{eq:logistic06} to render the joint conditional probability function as follows:
\begin{eqnarray}
\label{eq:logistic07}
\mathbb{P}(X \mid \alpha,  F, D, S)
&=&  \prod_{1\le i<j \le n}
\frac{e^{X_{ij}(\alpha + S_{ij} + f_i^T D f_j) }}{1 + e^{\alpha + S_{ij}+ f_i^T D f_j}} \nonumber \\
&=& \frac{\exp\left(\alpha \sum_{1\le i< j\le n}X_{ij} +\frac{1}{2} X \bullet (F^T D F) +\frac{1}{2} X \bullet S\right)}{\prod_{1\le i<j \le n}  \left(1 + e^{\alpha + f_i^T D f_j +S_{ij}}\right) }.
\end{eqnarray}

\section{Estimation}
\label{sec:estimate}
Note that in the model \eqref{eq:logistic07}, the log-likelihood function has the form as follows:
\begin{eqnarray}
\label{eq:logistic08}
\mathbb{L}(\alpha,  F, D, S; X)
&=& \alpha \sum_{1\le i< j\le n}X_{ij} +\frac{1}{2} X \bullet (F^T D F) +\frac{1}{2} X \bullet S \\
&& -\sum_{1\le i<j\le n} \log \left(1 + e^{\alpha + f_i^T D f_j +S_{ij} }\right). \nonumber
\end{eqnarray}

If we consider maximizing the above log-likelihood function,
we will encounter several technical issues that are described below.
\begin{enumerate}
\item We would like the matrix $S \in \mathbb{R}^{n \times n}$ to have as many zero entries as possible; i.e., matrix $S$ is {\it sparse.}

\item There is an identifiability issue with the formation $F^T D F$.
More specifically, let $P \in \mathbb{R}^{K \times K}$ be a signed permutation matrix, then we have $P^T P = I_n$, where $I_n \in \mathbb{R}^{K \times K}$ is the identity matrix.
Notice that matrix $F' = PF$ is also a factor loading matrix, and
matrix $D' = P D P^T$ is still a diagonal matrix;
we have
$$
F^T D F = F^T P^T P D P^T P F = (F')^T D' F',
$$
i.e., the choice of $F$ and $D$ is not unique.

\item We would like the number of nonzeros in each column of $F$ to be small, reflecting that each node is associated with a small number of underlying topics.

\item Overall, the rank of matrix $F^T D F$ cannot be larger than $\min\{n,K\}$.
With the application that we have in mind, in this paper, we assume that $K$ is much smaller than $n$.

\item Following the approaches that were mentioned in the Introduction, we propose to relax $F^T D F$ to $L$, where $L$ is a low rank matrix.
Furthermore, to ensure the separation of matrices $\alpha\mathbbm{1}\mathbbm{1}^T$ and an arbitrary matrix $L$, we assume that the eigen-vector of $L$ is centered, that is,
\[
    JLJ = L \quad \mbox{where} \quad
    J=I_{n}-\frac{1}{n}\mathbbm{1}\mathbbm{1}^T,
\]
where $\mathbbm{1}$ denotes a $n$-dimensional vector whose entries are all $1$'s. Since we have $L=F^TDF$, this condition uniquely identifies $F$ up to a common orthogonal transformation of its columns.

\end{enumerate}

Directly maximizing the objective function in \eqref{eq:logistic08} is not going to be an easy task.
Consequently, the log-likelihood function in \eqref{eq:logistic08} can be rewritten as
\begin{eqnarray}
\label{eq:logistic09}
\mathbb{L}_n(\alpha,  L, S; X)
&=& \alpha \sum_{1\le i< j\le n}X_{ij} +\frac{1}{2} X \bullet L +\frac{1}{2} X \bullet S \\
&& -\sum_{1\le i<j\le n} \log \left(1 + e^{\alpha + L_{ij} +S_{ij}}\right). \nonumber
\end{eqnarray}

We propose a penalized likelihood estimation approach as follows:
\begin{eqnarray}
\label{eq:logistic10}
(\hat{\alpha}, \widehat{L}, \widehat{S})
= \mbox{arg min}_{\alpha, L, S} \left\{
-\frac{1}{n} \mathbb{L}_n(\alpha,  L, S; X) + \gamma \|S\|_1
+ \delta \|L\|_\ast\right\},
\end{eqnarray}
where $\gamma>0$ and $\delta>0$ are algorithmic parameters whose values will be discussed later,
the $L_1$ norm of matrix $S$ is defined as $\|S\|_1 = \sum_{i\neq j} |S_{ij}|$ (Note that we do not penalize the diagonal entries of $S$), and nuclear norm of matrix $L$ is defined as $\|L\|_\ast = {\mbox{trace}\sqrt{(L^T L)}}$.
Recall that both $S$ and $L$ are symmetric matrices.
The entries of matrix $S$ can either be positive or negative.
Note that we have imposed the diagonal entries of the matrix $X$ to be zeros.
Given that $L = F^T D F$ where matrix $D$ is diagonal with nonnegative diagonal entries, it is easy to see that matrix $L$ is positive semidefinite; which consequently leads to $\|L\|_\ast = \mbox{trace}(L)$, which is a linear functional to the matrix $L$.
The nuclear norm of $L$ mimicks the number of nonzero eigenvalues of $L$, which is the same as the rank
of $L$.
The regularization based on the nuclear norm was proposed in \cite{fazel2001rank}
and its statistical properties are studied in \cite{bach2008consistency}.

After we have obtained $\widehat{S}$ in \eqref{eq:logistic10}, we can uncover the graphical model by investigating non-zero entries in $\widehat{S}$.
On the other hand, when we have calculated $\widehat{L}$, we may not be able to find binary matrix $F$ and nonnegative diagonal matrix $D$ such that $\widehat{L} = F^T D F.$
This is the price we have to pay for an amenable computational approach.
The rank of estimated $\widehat{L}$ will be our estimate of the number of factors (i.e., the number of underlying common topics).
We will discuss the issue on assigning the community membership of each node $i$ later in Section \ref{sec:Num-anal}.

\section{Non-asymptotic error bound of the estimator}
\label{sec:theorem}
In this section, we focus on investigating the behaviour of non-asymptotic error bound of our estimator in the context where the number of papers in a network is explicitly tracked. We are interested in solving the following optimization problem :
\begin{equation}\label{eq:Modified}
\min\limits_{\alpha \in R, S = S^T \atop L \succcurlyeq 0}
-\frac{1}{n} \log \prod_{1\le i,j\le n}\frac{\exp\left(X_{ij}\left(\alpha+
L_{ij}+S_{ij}\right)\right)}{1+\exp\left(\alpha+
L_{ij}+S_{ij}\right)} + \delta \|L\|_\ast + \gamma \|S\|_1.
\end{equation}
For the convenience of theoretical investigation, we slightly modify the first term in the objective function summing over all $(i,j)$ pairs.
After scaling, due to symmetry of $X$,$L$, and $S$, the only difference between (\ref{eq:logistic10}) and (\ref{eq:Modified}) is in the inclusion of terms in diagonal pairs $(i,i),\forall i=1,\dots,n$.
Note that we have $X_{ii}=0$ in our setting.

We borrow the idea of this modification from the work of \cite{ma2017exploration}, where they also consider the latent factor model in analyzing the embedded topics in the network but without the sparse component.
As stated in \cite{ma2017exploration}, this slight modification leads to neither theoretical consequence nor noticeable difference in practice.
Let ($\widehat{\alpha},\widehat{L},\widehat{S}$) be the solution to (\ref{eq:Modified}), and ($\alpha^{*},L^{*},S^{*}$) be the ground truth, which governs the data generating process.
Let $\widehat{\Theta}$ and $\Theta^{*}$ be defined respectively as $\widehat{\Theta} = \widehat{\alpha}\mathbbm{1}\mathbbm{1}^T+\widehat{L}+\widehat{S}$ and $\Theta^{*} = \alpha^{*}\mathbbm{1}\mathbbm{1}^T+L^{*}+S^{*}$.
And denote the error term for each parameter as $\widehat{\Delta}^{\Theta} = \widehat{\Theta}-\Theta^{*},
\widehat{\Delta}^{\alpha} = \widehat{\alpha}-\alpha^{*},
\widehat{\Delta}^L = \widehat{L}-L^{*},
\widehat{\Delta}^S = \widehat{S}-S^{*}.$
Throughout the discussion, let $P^{*}=\bigg\{\frac{\exp(\Theta_{ij}^{*})}{1+\exp(\Theta_{ij}^{*})}\bigg\}_{1 \leq i,j \leq n} \in \mathbb{R}^{n \times n}$.
We describe several assumptions before we can establish theoretical guarantees of our estimator.

\begin{assumption}(\textbf{Strong convexity})  \label{Ass:1}
For any $\Theta \in \mathbb{R}^{n\times n}$, define the log-likelihood in (\ref{eq:Modified}):
$$
h(\Theta) = -\frac{1}{n}\sum_{i,j} \big\{ X_{ij}\Theta_{ij} - \log(1+\exp(\Theta_{ij})) \big\}.
$$
We assume that $h(\Theta)$ is $\tau$-strongly convex in a sense that lowest eigenvalue of Hessian matrix of the log-likelihood function is bounded away from zero ($\tau > 0$):
\[
\nabla^{2}h(\Theta) = \mbox{diag}\Big(\mbox{vec}\Big(\frac{1}{n}\frac{\exp(\Theta)}{(1+\exp(\Theta))^{2}}
\Big)\Big) \succcurlyeq \tau I_{n^{2} \times n^{2}}.
\]
For any vector $a$, $\mbox{diag}(a)$ is the diagonal matrix with elements of $a$ on its diagonal.
For any matrix $B=[b_1,\dots,b_n]\in\mathbb{R}^{n \times n}$, $\mbox{vec}(B)\in\mathbb{R}^{n^2}$ is obtained by stacking $ b_1,\dots, b_{n}$ in order.
For any square matrix $A$ and $B$, we have $ A \succcurlyeq B $ if and only if matrix $A-B$ is positive semi-definite.
\end{assumption}

\begin{assumption} (\textbf{Identifiability of $\alpha\mathbbm{1}\mathbbm{1}^T$ and $L$}) \label{Ass:2}
To ensure the separation between $\alpha\mathbbm{1}\mathbbm{1}^T$ and $L$, we assume that the latent variables are centered, that is $JL=L$, where $J=I_{n}-\frac{1}{n}\mathbbm{1}\mathbbm{1}^T$, where $\mathbbm{1}$ denotes an all one vector in $\mathbb{R}^{n}$.
\end{assumption}

\begin{assumption} (\textbf{Spikiness of $L$ and Constraint on $\alpha$}) \label{Ass:3}
We impose a spikiness condition $\|L\|_{\infty}\leq\frac{\kappa}{\sqrt{n \times n}}$ on $L$, to ensure the separation of $L$ and matrix $S$ \cite{agarwal2012noisy}.
We would also like to note that the constraint $|\alpha|\leq C\kappa$, for an absolute constant $C$, is included partially for obtaining theoretical guarantees.
\end{assumption}

With these assumptions, we present the behavior of non-asymptotic error bound of our estimator through the following theorem. In our result, we measure error using squared Frobenius norm summed across three matrices:
\[
    e^{2}\big(\widehat{\alpha}\mathbbm{1}\mathbbm{1}^T,\widehat{L},\widehat{S}\big)
    :=\big\|\widehat{\Delta}^{\alpha}\mathbbm{1}\mathbbm{1}^T\big\|_{F}^{2} + \big\|\widehat{\Delta}^{L}\big\|_{F}^{2} + \big\|\widehat{\Delta}^{S}\big\|_{F}^{2}
\]

\begin{theorem} \label{Th:th1}
Under the Assumptions \ref{Ass:1}, \ref{Ass:2} and \ref{Ass:3},
if we solve the convex problem (\ref{eq:Modified}) with a pair of regularization parameter $(\delta,\gamma)$ satisfying
\begin{align} \label{eq:49}
\delta \geq 2 \left\|\frac{1}{n}(X-P^{*}) \right\|_{op} \quad and \quad \gamma \geq 2\left\|\frac{1}{n}(X-P^{*})\right\|_{\infty}+4\kappa\tau\bigg(\frac{Cn+1}{n} \bigg),
\end{align}
then there exist universal constants $c_{j}$, j = 1,2,3,  for all integers $k = 1,2,...,n$, and $s = 1,2,...,n^{2}$, and we have the following upper bound of $e^{2}\big(\widehat{\alpha}\mathbbm{1}\mathbbm{1}^T,\widehat{L},\widehat{S}\big)$:
\begin{align} \label{theorem1}
    e^{2}\big(\widehat{\alpha}\mathbbm{1}\mathbbm{1}^T,\widehat{L},\widehat{S}\big) \leq
    \underbrace{c_{1}\frac{\delta^{2}}{\tau^{2}}}
    _{\mathcal{K}_{\alpha^*}} +
    \underbrace{c_{2}\frac{\delta^{2}}{\tau^{2}}
    \bigg\{k + \frac{\tau}{\delta}\sum_{j=k+1}^{n}\sigma_{j}(L^{*})\bigg\}}_{\mathcal{K}_{L^*}} +
    \underbrace{c_{3}\frac{\gamma^{2}}{\tau^{2}}
    \bigg\{s + \frac{\tau}{\gamma}\sum_{(i,j) \notin M}|S^*_{ij}|\bigg\}}_{\mathcal{K}_{S^*}},
\end{align}
where $M$ is an arbitrary subset of matrix indices of cardinality at most $s$.
\end{theorem}

We would first like to note that the result presented in Theorem \ref{Th:th1} can be thought of as an extension of Theorem $1$ presented in paper \cite{agarwal2012noisy} to a generalized linear model.
Specifically, our work considers a logistic loss function whose parameter is characterized by a sparse matrix plus a low rank matrix,
whereas Agarwal, et al. \cite{agarwal2012noisy} work on a general linear observation model whose parameter is also
characterized by a sum of a low rank matrix and a sparse matrix.


Astute readers might have noticed that the upper bound in \eqref{theorem1} consists of three different terms, where we denote them as
$\mathcal{K}_{\alpha^*}$, $\mathcal{K}_{L^*}$ and $\mathcal{K}_{S^*}$.
Each respective term is involved with estimating three model parameters: $\alpha, L$ and $S$.
To be more specific, both $\mathcal{K}_{L^*}$ and $\mathcal{K}_{S^*}$ have two types of error:
1) The first one is called as an ``estimation error.''
This error represents the statistical cost of estimating parameters that belong to the model subspace.
2) Another quantity is referred as ``approximation error.''
This error occurs when we only focus on estimating parameters within the model subspace,
and it shrinks as the model subspace becomes large.

The result of the Theorem \ref{Th:th1} provides a family of upper-bounds, one for each indexed by a specific choice of model subspace $M$, and rank parameter $k$.
In other words, this means that the subset $M$ and the target rank $k$ can be adaptively chosen so as to obtain the tightest upper bound.
In ideal case where $L^*$ is an exact low rank matrix with rank $k$ \big(i.e., $\mbox{rank}(L^*)=k$\big) and $S^*$ is a sparse matrix, whose support lies within the model subspace $M$ \big(i.e., $\mbox{supp}(S^*) \subset M$\big), we can easily see ``\emph{approximation error}'' terms in $\mathcal{K}_{L^*}$ \big(i.e., $\delta\sum_{j=k+1}^{n}\sigma_j(L^*)$ \big) and in $\mathcal{K}_{S^*}$ \big(i.e., $\gamma\sum_{(i,j)\notin M}|S_{ij}^*|$\big) disappear, giving us Frobenius error bound as follows:
\[
    e^{2}(\hat{\alpha}\mathbbm{1}\mathbbm{1}^T,\widehat{L},\widehat{S}) \lesssim \delta^2(k+1)+\gamma^2s.
\]
Here we use the notation $X \lesssim Y$ to denote that there exists universal absolute constant $C$ such that $X \leq CY$.


\section{Computation}
\label{sec:compute}

We propose a method that takes advantage of the special structure
of the $L_1$ and the nuclear norm by means of the
alternating direction method of multiplier (ADMM), which is a method
that has recently gained momentum.
An examination of the objective function in \eqref{eq:logistic10} unvails that
terms
\[
\alpha \sum_{1\le i< j\le n}X_{ij} +\frac{1}{2} X \bullet L +\frac{1}{2} X \bullet S
\]
are linear in $\alpha,  L$, and $S$.
The term
\[
\sum_{1\le i<j\le n} \log \left(1 + e^{\alpha + L_{ij} +S_{ij}}\right)
\]
is convex with respect to $\alpha,  L$, and $S$.
Functions $\|S\|_1$ and $\|L\|_\ast$ are known to be convex functions.
Therefore, the objective function in \eqref{eq:logistic10} is convex.
The above convex optimization problem can be solved via ADMM as follows.

\subsection{ADMM approach}
\label{sec:ADMM}
We give a review of the alternating direction method
of multiplier (ADMM).
Consider two closed convex functions
$$
f : \chi_f \to \mathbb{R} \mbox{ and } g : \chi_g \to \mathbb{R},
$$
where the domain $\chi_f$ and $\chi_g$ of functions $f$ and $g$ are closed convex subsets of $\mathbb{R}^d$, and $\chi_f \bigcap \chi_g$ is nonempty.
Both $f$ and $g$ are possibly non-differentiable.
The alternating direction method
of multiplier is an iterative algorithm that solves the following generic optimization problem:
$$
\min_{x \in \chi_f \bigcap \chi_g} \left\{f(x) + g(x) \right\},
$$
or equivalently
\begin{eqnarray}
\label{eq:admm1}
\min_{x \in \chi_f, z\in \chi_g} & \left\{f(x) + g(z) \right\}, \\
\mbox{ subject to } & x = z. \nonumber
\end{eqnarray}
To describe the algorithm, we will need the following proximal operators
\begin{itemize}
\item $\mathbf{P}_{\lambda,f}: \mathbb{R}^d \to \chi_f$ as
$$
\mathbf{P}_{\lambda,f}(v) = \mbox{arg min}_{x \in \chi_f} \left\{
f(x) + \frac{1}{2\lambda} \|x-v\|^2_2
\right\},
$$

\item and $\mathbf{P}_{\lambda,g}: \mathbb{R}^d \to \chi_g$ as
$$
\mathbf{P}_{\lambda,g}(v) = \mbox{arg min}_{x \in \chi_g} \left\{
g(x) + \frac{1}{2\lambda} \|x-v\|^2_2
\right\},
$$
where $\|\cdot\|_2$ is the usual Euclidean norm on $\mathbb{R}^d$ and $\lambda$ is a scale parameter that is a fixed positive constant.
\end{itemize}
The algorithm starts with some initial values $x^0 \in \chi_f,
z^0 \in \chi_g, u^0 (=\lambda y^0) \in \mathbb{R}^d$.
At the $(m+1)$th iteration, $(x^m, z^m, u^m)$ is updated according to the following steps until convergence
\begin{itemize}
\item Step 1: $x^{m+1} = \mathbf{P}_{\lambda,f}(z^m - u^m)$,

\item Step 2: $z^{m+1} = \mathbf{P}_{\lambda,g}(x^{m+1} + u^m)$,

\item Step 3: $u^{m+1} = u^m + x^{m+1} - z^{m+1}$.

\end{itemize}
The convergence properties of the algorithm are summarized in the following result as in \cite{boyd2011distributed}.
Let $p^\ast$ be the minimal value in \eqref{eq:admm1}.

\begin{theorem}[Boyd et al., 2011]
Assume functions $f: \chi_f \to \mathbb{R}$ and
$g: \chi_g \to \mathbb{R}$ are
closed convex functions, whose domains $\chi_f$ and $\chi_g$
are closed convex subsets of $\mathbb{R}^d$ and
$\chi_f \bigcap \chi_g \neq \emptyset$.
Assume the Lagrangian of \eqref{eq:admm1}
$$
L(x,z,y) = f(x) + g(z) + y^T(x-z)
$$
has a saddle point, that is, there exists $(x^\ast, z^\ast, y^\ast)$ (not necessarily unique) that $x^\ast \in \chi_f$ and
$z^\ast \in \chi_g$, for which
$$
L(x^\ast, z^\ast, y) \le L(x^\ast, z^\ast, y^\ast) \le
L(x, z, y^\ast), \qquad \forall x, z, y \in \mathbb{R}^d.
$$
Then the ADMM has the following convergence properties.
\begin{enumerate}
\item Residual convergence. $x^m - z^m \to 0$
as $m \to \infty$; i.e., the iterates approach feasibility.

\item Objective convergence. $f(x^m) + g(z^m) \to p^\ast$ as $m \to \infty$; i.e., the objective function of
the iterates approaches the optimal value.

\item Dual variable convergence. $y^m \to y^\ast$ as $m \to \infty$, where $y^\ast$ is a dual optimal point.

\end{enumerate}
\end{theorem}

Now we describe how ADMM can be adopted to solve for our penalized likelihood estimation problem in \eqref{eq:logistic10}.
We reparameterize $M = L + S$ and let $x = (\alpha, M, L, S)$ (viewed as a vector).
We define the following:
\begin{eqnarray*}
\chi_f &=& \{ (\alpha, M, L, S): \alpha\in\mathbb{R}, M, L, S \in \mathbb{R}^{n \times n},
L \mbox{ is positive semidefinite, }
S \mbox{ is symmetricg} \}, \\
f(x) &=&
-\frac{\alpha}{n} \sum_{1\le i< j\le n}X_{ij}
-\frac{1}{2n} X \bullet M 
+ \frac{1}{n}
\sum_{1\le i<j\le n} \log \left(1 + e^{\alpha + M_{ij}}\right)
+ \gamma \|S\|_1
+ \delta \|L\|_\ast, \\
\chi_g &=& \{ (\alpha, M, L, S): \alpha\in\mathbb{R}, M, L, S \in \mathbb{R}^{n \times n},
M \mbox{ is symmetric and } M=L+S \}, \mbox{ and }\\
g(x) &=& 0, \mbox{ for } x \in \chi_g.
\end{eqnarray*}
One can verify that \eqref{eq:logistic10} can be written as
$$
\min_{x \in \chi_f \bigcap \chi_g} \left\{f(x) + g(x) \right\}.
$$

We now present each of the three steps of the ADMM algorithm and show that the proximal
operators $\mathbf{P}_{\lambda,f}$ and $\mathbf{P}_{\lambda,g}$ are easy to evaluate.
Let
$$
x^m = (x^m_\alpha, x^m_M, x^m_L, x^m_S), \quad
z^m = (z^m_\alpha, z^m_M, z^m_L, z^m_S), \quad
u^m = (u^m_\alpha, u^m_M, u^m_L, u^m_S).
$$
Step 1. We solve $x^{m+1} = \mathbf{P}_{\lambda,f}(z^m - u^m)$.
Due to the special structure of $f(\cdot)$,
$x^{m+1}_\alpha, x^{m+1}_M, x^{m+1}_L$, and $x^{m+1}_S$
can be updated separately.
More precisely, we have
\begin{eqnarray}
x^{m+1}_\alpha, x^{m+1}_M &=& \mbox{arg min}_{\alpha, M} \quad
-\frac{\alpha}{n} \sum_{1\le i< j\le n}X_{ij}
-\frac{1}{2n} X \bullet M
+ \frac{1}{n} \sum_{1\le i<j\le n} \log \left(1 + e^{\alpha + M_{ij}}\right) \nonumber \\
&& + \frac{1}{2\lambda}\left[\alpha - (z^m_\alpha - u^m_\alpha)\right]^2
+ \frac{1}{2\lambda}\|M - (z^m_M - u^m_M)\|^2_F, \label{eq:admm02} \\
x^{m+1}_L &=& \mbox{arg min}_{ L} \quad \delta \|L\|_\ast
+ \frac{1}{2\lambda}\|L - (z^m_L - u^m_L)\|^2_F,\label{eq:admm03}
\\
&& \mbox{subject to $L$ is positive semidefinite;} \nonumber \\
x^{m+1}_S &=& \mbox{arg min}_{ S} \quad  \gamma \|S\|_1
+ \frac{1}{2\lambda}\|S - (z^m_S - u^m_S)\|^2_F,
\label{eq:admm04}
\\
&& \mbox{subject to $S$ is symmetric,} \nonumber
\end{eqnarray}
where $\|\cdot\|_F$ is the matrix Frobenius norm, defined as
$\|M\|^2_F = \sum_{i,j} m^2_{ij}$ for a matrix $M = \{(m_{ij})\}_{i,j=1}^{n}$.
The problem in \eqref{eq:admm02} may not have a closed-form solution.
We use a simple gradient descent to solve in this step, setting the step size equal to $0.05$ and stopping criteria as $\max\big(|x_{\alpha,m}^{(t+1)}-x_{\alpha,m}^{(t)}|,\|x_{M,m}^{(t+1)}-x_{M,m}^{(t)}\|_{\infty}\big) \leq 10^{-9}$.
Note that there are close-form solutions to \eqref{eq:admm03} and \eqref{eq:admm04}, while \eqref{eq:admm02} is a unconstrained convex optimization problem.
More specifically, in \eqref{eq:admm03}, suppose the eigenvalue decomposition of the symmetric matrix $(z^m_L - u^m_L)$ can be written as
$$
z^m_L - u^m_L = T \Lambda T^T,
$$
where $T$ is orthogonal ($T T^T = I_n$). Then, for $J=I_{n}-\frac{1}{n}\mathbbm{1}\mathbbm{1}^T$, we have
$$
x^{m+1}_L = J \big(T \mbox{diag}(\Lambda-\lambda \delta\big)_+ T^T)J^{T},
$$
and diag$(\Lambda-\lambda \delta)_+$ is a diagonal matrix with the $j$th
diagonal entry being
$$
(\Lambda_{jj}-\lambda \delta)_+ = \left\{
\begin{array}{ll}
0, & \mbox{ if } \Lambda_{jj} < \lambda \delta, \\
\Lambda_{jj}-\lambda \delta, & \mbox{ if } \Lambda_{jj} \ge \lambda \delta .
\end{array}
\right.
$$
In \eqref{eq:admm04}, we have, for $i \neq j$,
\[
S_{ij} = \left\{
\begin{array}{ll}
0, & \mbox{ if } |(z^m_S - u^m_S)_{ij}| < \lambda \gamma,  \\
(z^m_S - u^m_S)_{ij}-\lambda\gamma , &
\mbox{ if } (z^m_S - u^m_S)_{ij} > \lambda \gamma, \\
(z^m_S - u^m_S)_{ij}+\lambda\gamma , &
\mbox{ if } (z^m_S - u^m_S)_{ij} <- \lambda \gamma.
\end{array}
\right.
\]

\noindent
Step 2. We solve $z^{m+1} = \mathbf{P}_{\lambda,g}(x^{m+1} + u^m)$.
A close-form solution exists here.
Denote
$
\bar{\alpha} = x^{m+1}_\alpha + u^m_\alpha,
\bar{M} = x^{m+1}_M + u^m_M,
\bar{L} = x^{m+1}_L + u^m_L$, and
$\bar{S} = x^{m+1}_S + u^m_S,
$
then evaluating $\mathbf{P}_{\lambda,g}(x^{m+1} + u^m)$ becomes
\begin{eqnarray*}
\min_{\alpha,M,L,S} & \quad
\frac{1}{2}[\alpha - \bar{\alpha}]^2
+ \frac{1}{2}\|M - \bar{M}\|^2_F
+ \frac{1}{2}\|L - \bar{L}\|^2_F
+ \frac{1}{2}\|S - \bar{S}\|^2_F  \\
\mbox{ subject to } & M \mbox{ is symmetric and } M=L+S.
\end{eqnarray*}
The above optimization problem has a close-form solution, which is as follows:
\begin{eqnarray*}
z^{m+1}_\alpha &=& \bar{\alpha}, \\
z^{m+1}_M &=&
\frac{1}{3} \bar{M} + \frac{1}{3} \bar{M}^T + \frac{1}{3} \bar{L} + \frac{1}{3} \bar{S}, \\
z^{m+1}_L &=&
\frac{1}{6} \bar{M} + \frac{1}{6} \bar{M}^T + \frac{2}{3} \bar{L} - \frac{1}{3} \bar{S},
\quad \text{and} \\
z^{m+1}_S &=&
\frac{1}{6} \bar{M} + \frac{1}{6} \bar{M}^T - \frac{1}{3} \bar{L} + \frac{2}{3} \bar{S}.
\end{eqnarray*}

\noindent
Step 3. We solve $u^{m+1} = u^m + x^{m+1} - z^{m+1}$, which is a simple arithmetic. \\

The most important implementation details of this algorithm are the choice of $\lambda$ and stopping criterion.
In this work, we simply choose $\lambda = 0.5$.
We terminate the algorithm  when in the $m$th iteration, we have $\|x^{m}_M - x^{m}_L - x^{m}_S\|_F \leq \delta$, with $\delta=10^{-7}$.

\section{ Numerical analysis and Applications}
\label{sec:Num-anal}
Section \ref{sec:Num-anal} is divided into two parts.
In Section \ref{sec:Syn-dat}, we conduct an empirical study of our proposed method with synthetic graphical structures.
In Section \ref{sec:citation_network}, we perform a real data analysis with a citation network for statisticians.


\subsection{Numerical experiments with synthetic data}
\label{sec:Syn-dat}
First, we introduce two synthetic scenarios that we want to explore (Section \ref{sec:Syn-set}).
Then, we describe three model selection criteria and four evaluation metrics for the selected model
(Section \ref{sec:Eval-crit}).
Subsequently, we elaborate experimental results from the synthetic networks and several interesting findings from those results (Section \ref{sec:obs}).

\subsubsection{Synthetic Setting} \label{sec:Syn-set}
Before specifying the network settings in two scenarios that we want to explore, let us first describe a set of steps for setting the model parameters, $\alpha^*, F^*, D^*$ and $S^*$ sequentially.
We put astroids in the superscripts of parameters to indicate that they are the ground truth.
Readers can refer the meaning of each parameter in the model in Section \ref{sec:intro} and  \ref{sec:model-form}.

\begin{enumerate}
    \item We draw an intercept term $\alpha^*$ in the logistic regression model from the uniform distribution that is supported on [-11,-10].
        In this way, we can make $\alpha^*$ have the least effects in creating edges in the network.

    \item Recall that the binary factor loading matrix $F^*$ encodes the relation between topics and papers
        (i.e., if $i$th paper studies about $k$th topic, we denote $F^*_{ki}=1$, otherwise $F^*_{ki}=0$).
        First, we assume that there are $n$ papers in the network, and $K$ topics are embedded in it.
        Each of them consists of roughly $\frac{n}{K}$ papers.
        This can be expressed in $F^*$ as follows:
        \[
            F^* =
               \begin{bmatrix}
                 \smash[b]{\block{n/K}} \\
                 && \smash[b]{\block{n/K}} \\
                 &&&& \ddots \\
                 &&&&& \block{n/K}
                \end{bmatrix}
              \in R^{K \times n},
        \]
        where each row of $F^*$ has $\frac{n}{K}$ $1$'s and each column has only one $1$.
        Note that the remaining entries of the matrix are filled with zeros.
        Then, we assume that $n_{l}$ papers share $l$ topics and $n_{m}$ papers share $m$ topics ($1 < l < m \leq K$).
        This can be represented in the $F^*$ in following steps:
        $1)$ Pick distinct $n_{l}$ indices randomly from $\{1,2,\dots,n\}$.
        We will denote the set of the indices as $\Omega_{n_l}$.
        $2)$ Choose $n_{m}$ indices from the set $\{1,2,\dots,n\}\setminus \Omega_{n_{l}}$ and denote the set of those indices as $\Omega_{n_{m}}$.
        $3)$ Make the columns of $F^*$ with corresponding indices in set $\Omega_{n_{l}} \cup \Omega_{n_{m}}$ zeros.
        $4)$ We use a notation $f^*_j$ to denote the $j$th column of the matrix $F^*$.
        Fill arbitrary $l$ entries of $f^*_j$ for $j \in \Omega_{n_l}$ with $1$s, and also fill arbitrary $m$ entries of $f^*_j$ for $j \in \Omega_{n_m}$ with $1$s.
        Lastly, we set $F^*=JF^*$ where $J=I_{n}-\frac{1}{n}\mathbbm{1}\mathbbm{1}^T$.

    \item Generate the weight coefficients of the factors $D^*_{ii}$ from the uniform distribution that is supported on $[19,20]$, $ \forall 1 \leq i \leq K$.
        In this way, we can leave the papers with same topic cluster together.

    \item Recall that the positive entries of $S^*$ can characterize the links in the network, which cannot be accounted by the common topics.
        In the spirit of this notion, we construct ad-hoc links in a way that they connect the clusters of papers with one topic.
        Specifically, we assume that there are $\frac{|S^*|}{\binom{K}{2}}$ edges between two clusters, where $|S^*|$ denotes the number of non-zero entries of the upper-triangular part of the matrix $S^*$.
        This can be implemented via the following steps:
        1) We construct $K$ sets $C^*_1,\dots,C^*_K$ that are defined as follows:
        \begin{align*}
           &C^*_{1} \subseteq \big\{1, 2, \dots, \frac{n}{K} \big\} \setminus \big\{ \Omega_{n_l}
           \cup \Omega_{n_m} \big\} \\
           &C^*_{2} \subseteq \big\{\frac{n}{K}+1, \dots, \frac{2n}{K} \big\} \setminus \big\{ \Omega_{n_l} \cup \Omega_{n_m} \big\}  \\
           & \quad \vdots \\
           &C^*_{K} \subseteq \big\{\frac{(K-1)n}{K}+1, \dots, n \big\} \setminus \big\{ \Omega_{n_l}
           \cup \Omega_{n_m} \big\}
        \end{align*}
        where each of them has arbitrary $\frac{|S^*|}{\binom{K}{2}}$ elements.
        2) Create a set $I_{S^*}$ whose elements are pairs of indices such that
        \[
            I_{S^*} =
            \bigg\{ (i_r,j_r) : i_r \in C^*_p, j_r \in C^*_q, 1 \leq p < q \leq K,   r = 1,2,\dots,\frac{|S^*|}{\binom{K}{2}} \bigg\},
        \]
        where we use $i_r$ to indicate $r$th element $i$ in the set.
        3) Draw $S^*_{ij} \sim \mbox{Unif}[19,20], \forall (i,j) \in I_{S^*}$.
        4) Lastly, make it symmetric by setting $S^*_{ji}=S^*_{ij},\forall 1\leq i < j \leq n$.
    \item Create an upper-triangular part of the adjacency matrix $X$ whose each entry $X_{ij}$ follows Bernoulli distribution.
        The distribution's parameter is parametrized by a probability, $P^{*}_{ij}=\frac{\exp (\alpha^{*}+f^{*T}_{i}D^{*}f^{*}_{j} +  S^{*}_{ij})}{1+\exp(\alpha^{*}+f^{*T}_{i}D^{*}f^{*}_{j} +  S^{*}_{ij})}$.
        After drawing all the entries of $X$ in the upper-triangular part, then make the matrix symmetric by setting $X_{ji}=X_{ij},\forall 1 \leq i < j \leq n$.
\end{enumerate}

With these settings in mind, we consider two scenarios, where each of them has three synthetic networks.

\begin{enumerate}
    \item In the first scenario, we consider three networks, in which each of them consists of papers with only one topic.
        Specifically, following three networks
        $\{(n^{(i)},n_{1}^{(i)},K^{(i)},|S^*|^{(i)})\}_{i=1}^{3}$ $=\{(30,30,3,9), (80,80,4,18),$ $(120,120,5,30)\}$ are considered.
        If we take an example, the notation  $(n^{(1)},n_{1}^{(1)},K^{(1)},|S^*|^{(1)})=(30,30,3,9)$ means that we generate a network with $30$ papers. There are $3$ topics embedded in the network, and $9$ random ad-hoc links connect $3$ clusters of papers, where each cluster represents a collection of papers with same topic.

    \item In the second scenario, we consider three networks, in which each of them has some papers that have more than one topic.
        In particular, we consider $\{(n^{(i)},n_{2}^{(i)},n_{3}^{(i)},K^{(i)},|S^{*}|^{(i)})\}_{i=1}^{3}=
        \{(120,0,10,3,18),$ $(210,50,0,3,18),\\(210,10,10,3,18)\}$.
        For instance, in the third case, we have a network with $210$ papers in total. There are $3$ topics commonly shared across the network.
        Among $210$ papers, $10$ papers randomly share $2$ topics out of $3$, other $10$ papers have $3$ mixed topics, whereas remaining $190$ papers only discuss $1$ topic.
        Note that the $3$ clusters from these $190$ papers are connected through $18$ random ad-hoc links.
\end{enumerate}
\begin{figure}[htbp]
\includegraphics[width=1\textwidth]{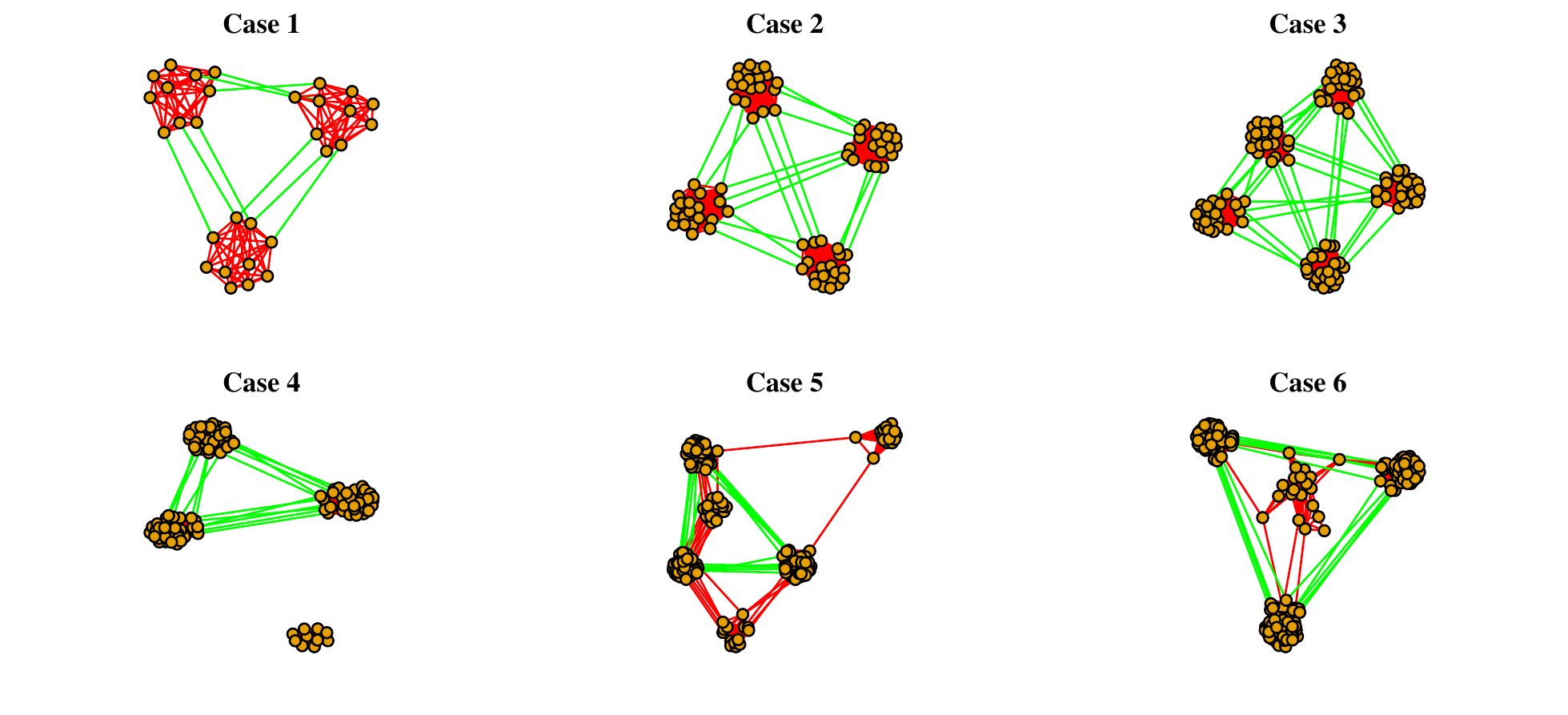}
\caption{Graphical illustrations of six synthetic networks.
Nodes that share the common factors are clustered. The cross cluster links are ad-hoc Citations.
All the graphs are drawn via the algorithm in \cite{fruchterman1991graph}.}
\label{fig:figure1}
\end{figure}

All six networks that are elaborated in scenarios $1$ and $2$ are visualized in Fig.\ref{fig:figure1}.
Notice that the nodes that share the common topics are clustered,
and the cross clustered links are the ad-hoc citations.

\subsubsection{Choosing the tuning parameters and evaluation criteria} \label{sec:Eval-crit}
Choosing a good pair of tuning parameters is an important yet challenging issue in our setting.
Here we present a heuristic procedure for choosing a good pair of tuning parameters $(\gamma,\delta)$.
Following the scree-plot approach in Ji and Jin \cite{ji2016coauthorship}, we plot the largest $15$ eigenvalues of the adjacency matrix $X$, and find an ``elbow'' point where the eigenvalues seem to level off.
An index of the point, which is to the left of this elbow point, is considered as the number of the topics embedded in the network.
(We will denote this number as $\widehat{K}^{\text{Scree}}$.)
We want to note that the scree-plot analysis serves as a good guideline for determining the range of grids to search over.
With the estimate of the number of topics in the network in mind, we record the $\text{rank}(\widehat{L}^{\gamma,\delta})$ and $|\widehat{S}^{\gamma,\delta}|$ (i.e., the cardinality of the set $\widehat{S}^{\gamma,\delta}$) for each tuning parameter pair on a given grid.
We need to go through several iterations of this recording procedure to find a proper range of grid, in which we can get $\text{rank}(\widehat{L}^{\gamma,\delta})=\widehat{K}^{\text{Scree}}$
and $10^{-4}\|X\|_0 \leq |\widehat{S}^{\gamma,\delta}| \leq 10^{-1}\|X\|_0$, via adjusting the range of grid for $\gamma$ and $\delta$ repeatedly.
Here $\|X\|_0$ denotes the number of nonzero entries in the matrix $X$.
Once we find a grid, which satisfies above constraints, we choose a pair of tuning parameters: 
\[
    (\gamma^{\text{Heu}},\delta^{\text{Heu}}):=
    \big\{(\gamma,\delta): \text{rank}(\widehat{L}^{\gamma,\delta})=\widehat{K}^{\text{Scree}},
    \mbox{ mode } |\widehat{S}^{\gamma,\delta}|: \mbox{ subject to } 10^{-4}\|X\|_0 \leq |\widehat{S}^{\gamma,\delta}| \leq 10^{-1}\|X\|_0 \big\}.
\]
One might wonder how the traditional model selection methods, such as the Bayes Information Criterion (BIC;\cite{schwarz1978estimating}) and the Akaike information criterion (AIC), work.
Recall that BIC and AIC are defined as follows:
\[
\mbox{BIC}(M) = -2  \mathbb{L}_n( \hat{\beta}(M)) + |M| \log \bigg(\frac{n(n-1)}{2}\bigg),
\]
and
\[
\mbox{AIC}(M) = -2  \mathbb{L}_n( \hat{\beta}(M)) + 2|M|.
\]
Here $M$ indicates the current model, which is implicitly understood that the model is obtained from certain tuning parameter pair $(\gamma,\delta)$.
We use $\mathbb{L}_n( \hat{\beta}(M))$ to denote the maximal log-likelihood for a given model $M$,
and $|M|$ is the number of free parameters in $M$, which is determined by the number of non-zeros in $\widehat{S}^{\gamma,\delta}$ and the low-rank matrix $\widehat{L}^{\gamma,\delta}$.
In detail, if we have $\text{rank}(\widehat{L}^{\gamma,\delta})=K$, we can establish the following
\[
|M| = \sum_{i < j}1_{\{S_{ij} \neq 0\}} + n K - \frac{K(K-1)}{2} + 1 ;
\]
since the number of free parameters in $\widehat{L}^{\gamma,\delta}$ is $K$ plus $nK - K(K+1)/2$, which is the number of free parameters in determining $K$ orth-normal vectors.
Additional $1$ in the last term is due to $\hat{\alpha}$.
We want to find a pair $(\gamma,\delta)$, which minimizes BIC($M$) or AIC($M$) as a function of $(\gamma,\delta)$, respectively, where we denote them as follows:
\[
    (\gamma^{BIC},\delta^{BIC}):= \mbox{arg min}_{\gamma, \delta}
    \mbox{BIC}(M),
    \quad
    (\gamma^{AIC},\delta^{AIC}):= \mbox{arg min}_{\gamma, \delta}
    \mbox{AIC}(M).
\]
We evaluate the models that are selected via our heuristic approach, BIC, and AIC by using the following four evaluation metrics:
\begin{eqnarray*}
M_1 &=& \mathbbm{1}\big\{\mbox{rank}(\widehat{L}) = \mbox{rank}(L^{*})\big\}, \\
M_2 &=& \frac{\left|\big\{(i,j):i<j:S^{*}_{i,j} \neq0 \quad \& \quad \widehat{S}_{i,j} \neq 0 \big\}\right|}{\left|\big\{(i,j):i<j:S^{*}_{i,j}\neq0\big\}\right|}, \\
M_3 &=& \frac{\left|\big\{(i,j):i<j:S^{*}_{i,j} = 0 \quad \& \quad \widehat{S}_{i,j} \neq 0 \big\}\right|}{\left|\big\{(i,j):i<j:S^{*}_{i,j}=0\big\}\right|}, \\
M_4 &=& \frac{\left|\big\{\mbox{Misclassified Nodes}\big\}\right|}{n},
\end{eqnarray*}
where $M_1$ is a metric on whether the selected model recovers the true low rank structure of network,
$M_2$ evaluates the positive selection rate of the sparse ad-hoc structure in network,
$M_3$ evaluates the false discovery rate of ad-hoc edges, and
$M_4$ calculates the proportion of mis-classified nodes to the entire nodes in the network.
With properly selected tuning parameter, $M_1$ will be 1, $M_2$ will be close to 1, and $M_3$ and $M_4$ will get close to 0.
We present the evaluation results of the six cases via the four criteria, $M_1,M_2,M_3$, and $M_4$ in Table.\ref{tab:table1}.

\begin{table}[htbp]
\centering
    \begin{tabular}{c|ccc|ccc|ccc}
                & \multicolumn{9}{c}{\textbf{Scenario 1}}                                              \\
                \cline{2-10}
                & \multicolumn{3}{c|}{Case 1} & \multicolumn{3}{c|}{Case 2} & \multicolumn{3}{c}{Case 3} \\
                \cline{2-10}
                & Heuristic   & AIC   & BIC  & Heuristic   & AIC   & BIC  & Heuristic   & AIC   & BIC  \\
                \hline
    $M_{1}$   &    1 (3)  &    0 (2)   &  0 (2)  &    1 (4)    &  0 (3)   &  0 (3)  &    1 (5)  &   0 (4)    & 0 (4) \\
    $M_{2}$   &    1     &    0 &      0 &         1    &     1  &    1  &         1    &    1   &   1   \\
    $M_{3}$   &    0     &    0 &      0 &         0    &     0  &    0  &         0    &    0   &   0   \\
    $M_{4}$   &  0  &    10/30    &  10/30     &   0   &   20/80    &   20/80   &  0      &  24/120     &   24/120    \\
                \hline
    \end{tabular}

    \begin{tabular}{c|ccc|ccc|ccc}
                & \multicolumn{9}{c}{\textbf{Scenario 2}}                                              \\
                \cline{2-10}
                & \multicolumn{3}{c|}{Case 4} & \multicolumn{3}{c|}{Case 5} & \multicolumn{3}{c}{Case 6} \\
                \cline{2-10}
                & Heuristic   & AIC   & BIC  &    Heuristic   & AIC   & BIC  & Heuristic   & AIC   & BIC  \\
                \hline
    $M_{1}$   &   1 (3)     &     1 (3)  &  1 (3)    &      1 (3)   &   1 (3)    &  1 (3)    &   1 (3)     &  1 (3)   & 1 (3)   \\
    $M_{2}$   &   17/18   &   0    &    0  &   17/18     &   0    &  0    &  16/18    &   0    &  0    \\
    $M_{3}$   &   0 &     0  &   0   &        0     &    0   &  0    &        0     &   0    & 0 \\
    $M_{4}$   &  0  &   0    &  0    &   0          &   0    &   0   &  0      &  0     &   0   \\
                \hline
    \end{tabular}
    \caption{ For two scenarios, our heuristic method chooses the model with $\widehat{L}$ with true rank, $\widehat{S}$ whose $M_2$ value is close to $1$, and $M_3$ value is close to $0$.
    Also note that it chooses a model whose mis-classification rate is close to $0$.
    A number in the parentheses represents the rank of $\widehat{L}$ estimated from ($\gamma^{\text{Heu}},\delta^{\text{Heu}}$), ($\gamma^{\text{AIC}},\delta^{\text{AIC}}$) and ($\gamma^{\text{BIC}},\delta^{\text{BIC}}$) for each case.}
    \label{tab:table1}
\end{table}

\subsubsection{Several Observations} \label{sec:obs}
\begin{enumerate}
    \item \textbf{Model Selection.}
    First and foremost, choosing a good pair of tuning parameters is critical when it comes to making a good statistical inference on data.
    As presented in Table.\ref{tab:table1}, both BIC and AIC, which are well known for their model selection consistency in asymptotic setting, appear to under-estimate both the number of topics and the number of ad-hoc links in the networks in our synthetic settings. This may be caused by the fact that these traditional methods take the sample size into account, and therefore penalizes the model complexity too harshly.
    In the heuristic approach, scree-plot plays an important role when it comes to recovering the number of topics, and this strategy leads us to good model selection results for all the six cases considered in two scenarios. (See Fig.$\ref{fig:figure2}$)

    \begin{figure}[htbp]
        \includegraphics[width=1\textwidth]{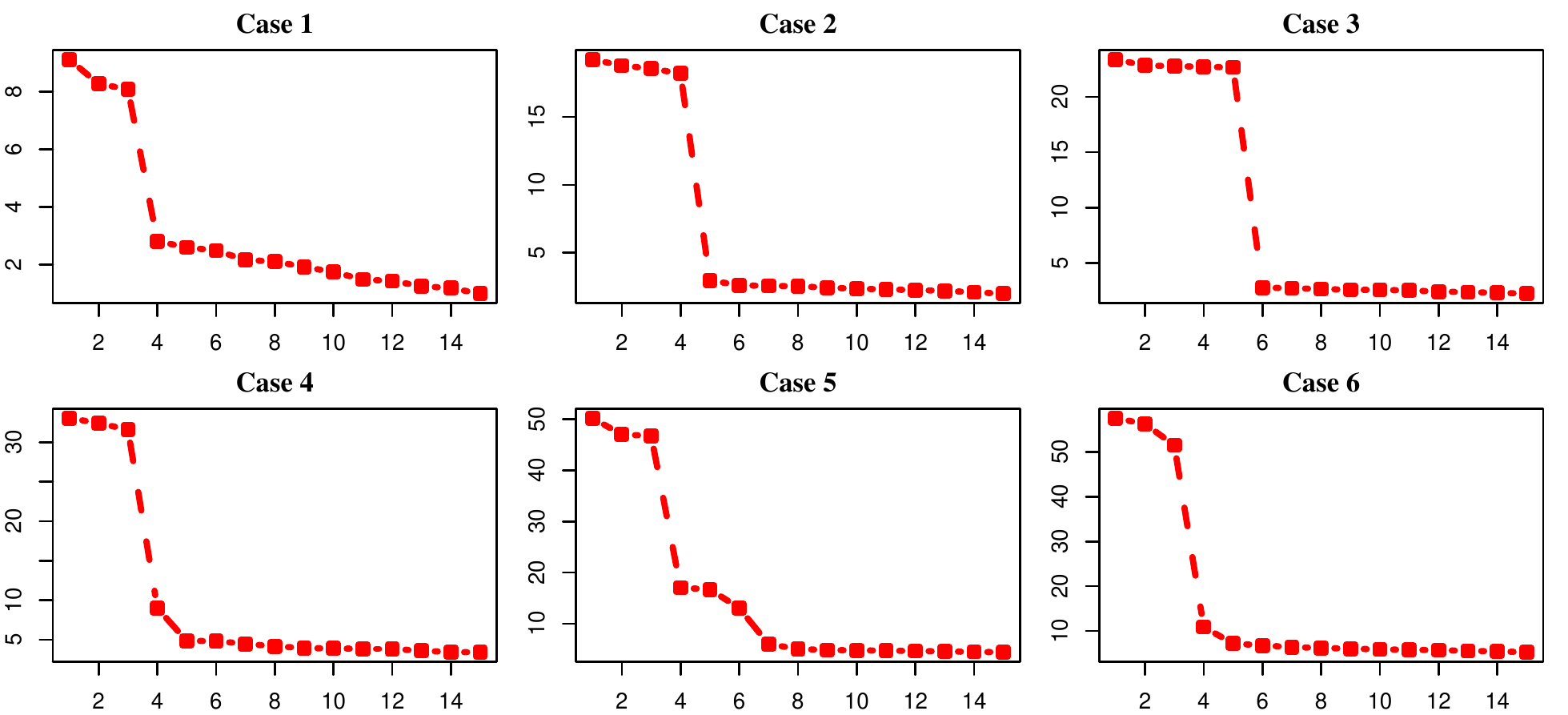}
        \caption{ Scree plots for six synthetic Networks.
        $\widehat{K}^{\mbox{Scree}}$ recovers the number of topics in the network correctly for all six cases.}
        \label{fig:figure2}
    \end{figure}

    \item \textbf{Node Membership.}
    After fitting the model with a proper pair of tuning parameters, $(\gamma,\delta)$, we need to determine whether the $i$th paper belongs to the $k$th topic or not.
    We apply a simple $k$-means clustering algorithm on the $\widehat{L}$ matrix's  $K$ eigenvectors where $K$ denotes the rank of matrix $\widehat{L}$.
    For the three cases in Scenario $1$, where each paper in the network only belongs to one topic, we confirm that $k$-means algorithm performs well on classifying papers in the network.
    However, in Scenario $2$ where we allow the papers in the network can have more than one topics, naive implementation of the $k$-means algorithm entails a problem -- it is not able to cope with the overlapped membership of nodes.
    In this case, we create a matrix $\widehat{E}_{K}\in \mathbb{R}^{n \times K}$, whose $i$th column corresponds to the $i$th eigenvector of the matrix $\widehat{L}$.
    In order to obtain a sense on how many clusters of papers exist in the latent space, we project each row of the $\widehat{E}_{K}$ on the first and second principal components of data matrix $\widehat{E}_{K}$, and plot the projected points on a $2$-dimensional plane.
    We count the number of distinct clusters plotted on the plane.
    Subsequently, we run the $k$-means algorithm on the projected points.
    Table.\ref{tab:table1} and Fig.\ref{fig:figure3} illustrate the result of the above procedure.
    They result seem to be consistent with the underlying truth.

    \begin{figure}[htbp]
    \centering
    \includegraphics[width=1\textwidth]{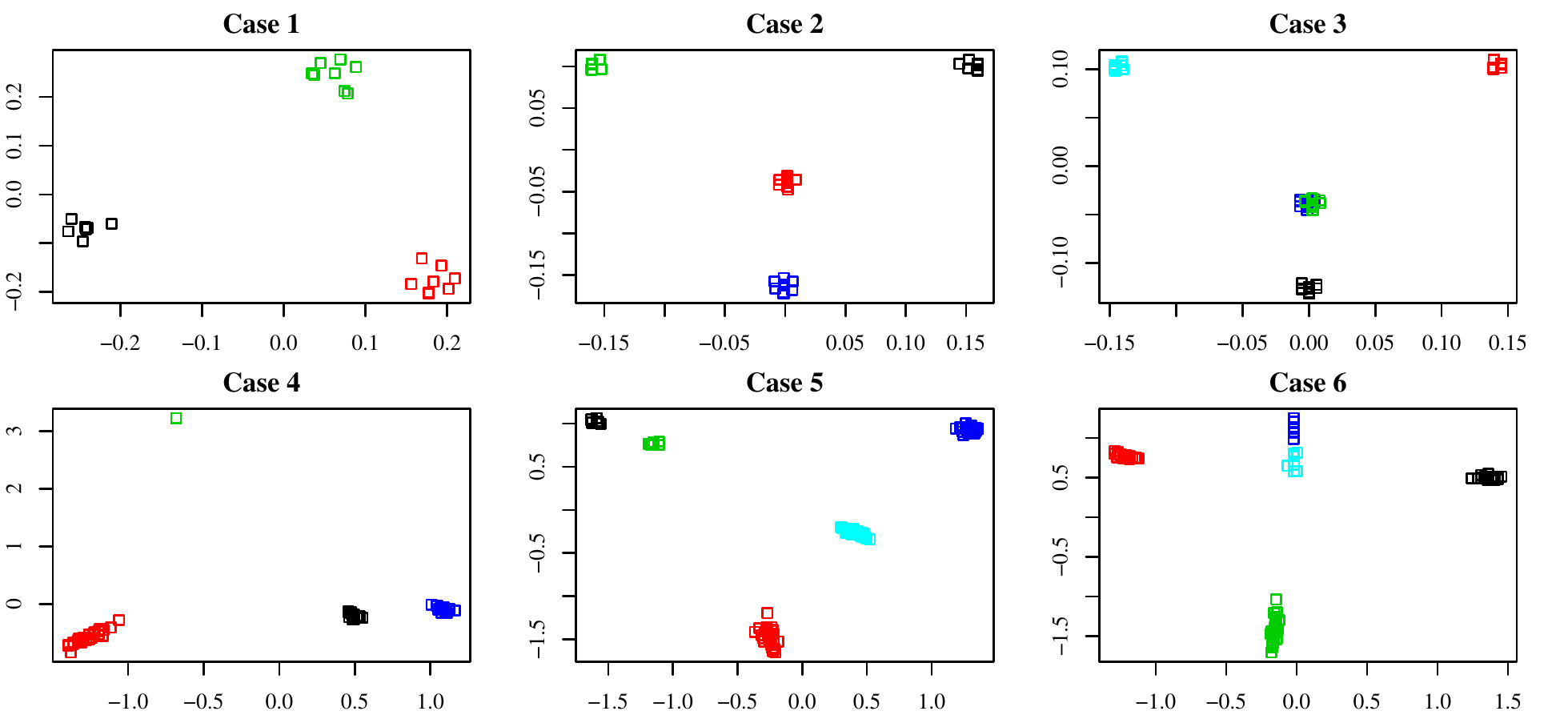}
    \caption{ Case $1\sim3$ : Plots of rows from the first two eigenvectors of $\widehat{L}^{\gamma^{\text{Heu}},\delta^{\text{Heu}}}$.
    Case $4 \sim 6$ : Plots of the projected points on the first (X-axis) and second (Y-axis) principal component of the data matrix $\widehat{E}_{K}^{\gamma^{\text{Heu}},\delta^{\text{Heu}}}$. Different colors represent different clusters of papers that $k$-means algorithm assigns.}
    \label{fig:figure3}
    \end{figure}

\end{enumerate}

\subsection{Citation networks for statisticians}
\label{sec:citation_network}
\begin{figure}[htbp]
\includegraphics[width=1\textwidth]{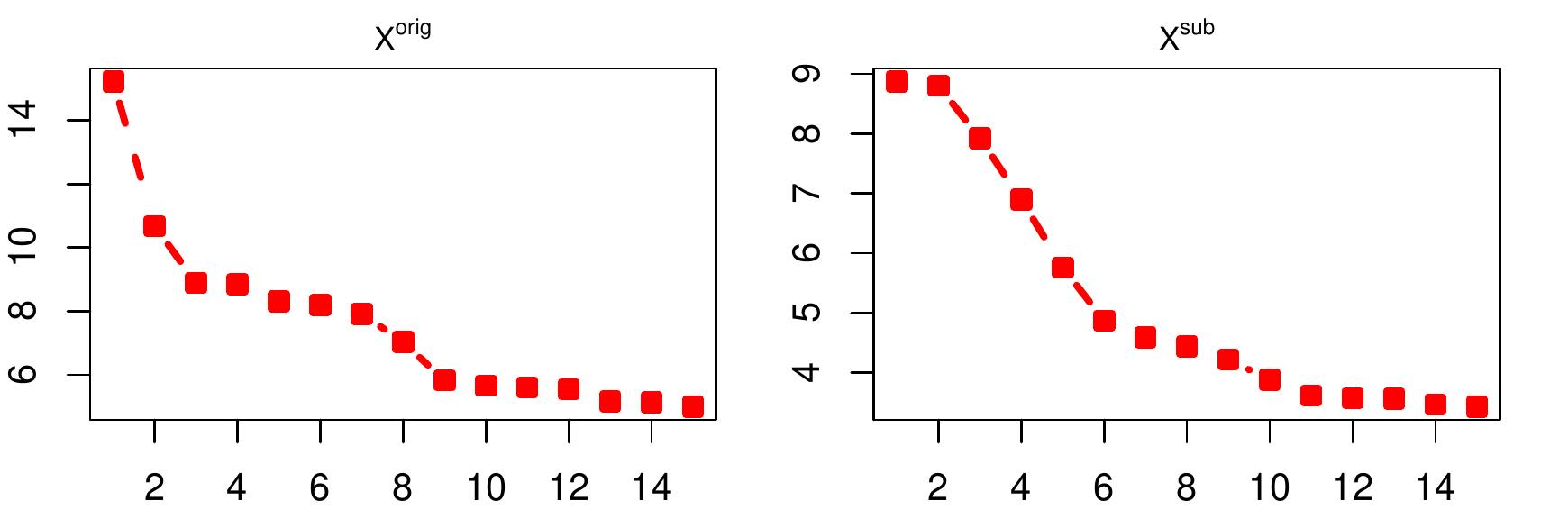}
\caption{From left to right : Scree plots of the adjacency matrix $X^{\mbox{orig}}$ and $X^{\mbox{sub}}$.}
\label{fig:figure4}
\end{figure}
Recently, Ji and Jin \cite{ji2016coauthorship} published an interesting dataset on citation network of papers from statistics journals.
Specifically, this dataset is based upon all papers published from $2003$ to the first half of $2012$, from the four top statistical journals: Annals of Statistics, Biometrika, Journal of American Statistical Association, and Journal of Royal Statistical Society (Series B).
Citational relationships of $3248$ papers are given in the form of adjacency matrix.
In our analysis, we focus our attentions on the papers which have greater than or equal to $10$ citational edges in the network of Ji and Jin \cite{ji2016coauthorship}.
After collecting papers with greater than or equal to $10$ citational edges and eliminating those that have no connecting edges from the rest, we have $232$ papers in total.

We denote the adjacency matrix of these $232$ papers as $X^{\mbox{orig}}$.
Elbow points of the scree plot from $X^{\mbox{orig}}$ may be at the $3$rd, $5$th, or $9$th largest eigenvalue, suggesting that there are from $2$ to $8$ embedded topics in the network. (See Fig. \ref{fig:figure4})
In light of this, we conduct the analysis in the following two steps:
\begin{enumerate}
    \item First, we assume that the network $X^{\mbox{orig}}$ has $2$ distinct topics and one giant mixed-component, which has a sub-network structure. Under this assumption, we set $\widehat{K}^{\mbox{scree}}$ as $3$, and select a proper model via our heuristic method.
        Then, we perform a $k$-means algorithm on matrix $\widehat{E}_{3}$ treating each row of the matrix as one data point.
        Note that we set the total number of clusters in the network as $3$ when we run the clustering algorithm.
    \item Next, we restrict the network to the giant component ignoring all the edges to/from outside and obtain a subnetwork.
        We denote the adjacency matrix of this subnetwork as $X^{\mbox{sub}}$. We set $\widehat{K}^{\mbox{scree}}$ as $5$, and also select a proper model through the heuristic method.
        Here, we run the $k$-means algorithm on $\widehat{E}_{5}$ setting the number of clusters as $5$.
\end{enumerate}

In the first step, a pair of parameters, $(\gamma^{\text{Heu}},\delta^{\text{Heu}}) = (0.0021094, 0.01913)$, gives us $\widehat{L}$ with rank $3$, and $\widehat{S}$ with $|\widehat{S}|=51$.
We list the first two topics discovered through our analysis.
\begin{itemize}
    \item Variable selection (VarSel), which includes $43$ paper.
    \item Multiple Hypothesis Testing (MulT), which includes $31$ papers.
\end{itemize}
The first topic studies on Variable Selection with high-dimensional data.
The second topic discusses Controlling False Discovery Rate in various statistical settings.
The third group, which consists of $158$ papers, is hard to interpret and appears to have sub-network structures.
For further investigation, we set this group as a giant component in the network, and denote the corresponding component's adjacency matrix as $X^{\mbox{sub}}$.
We perform a model selection as described in aforementioned Step $2$.
A pair of tuning parameters, $(\gamma^{\text{Heu}},\delta^{\text{Heu}}) = (0.00312, 0.0146)$, gives us the model with $\widehat{L}$ with rank $5$, and $\widehat{S}$ with $|\widehat{S}|=88$,
and we can obtain five sub-communities as follows:
\begin{itemize}
    \item Non-parametric Bayesian Statistics (NonPar), which includes $15$ papers.
    \item Functional / Longitudinal Data analysis (FuncAn), which includes $16$ papers.
    \item Dimension Reduction (DimRed), which includes $14$ papers.
    \item High-dimensional Covariance Estimation (CovEst), which includes $15$ papers.
    \item Mixed Topics (Mixed), which includes $98$ papers.
\end{itemize}
From the sub-network $X^{\mbox{sub}}$, we got four meaningful topics: Bayesian Statistics, Functional/Longitudinal Data Analysis, Dimension Reduction, and High-dimensional Covariance Estimation.
Due to the small volume of each community, we could manually check that the false discovery for each community is all zero.
(Full list of papers for each community is provided in {\em https://sites.google.com/site/namjoonsuh/publications}.)

The sub-network structure has also a big collection of papers that we refer to as the ``Mixed Topics'' cluster.
Not only could we see the papers with topics on Learning Theory, Non-parametric / Semi-parametric Statistics, Spatial Statistics, Theoretical Machine Learning, which does not seem to belong to any of the five communities listed above, but also we could identify the papers with combinations of two or three topics.
Papers, such as The Bayesian Lasso (T. Park, et al. $2008$), Coordinate-independent sparse sufficient dimension reduction and variable selection (X. Chen, et al. $2010$), are the examples of these papers.
It is also interesting to think about a reason on papers that seem to have obvious membership in one of $5$ communities other than Mixed Topic classified as Mixed Topic.
For instance, the paper, On the ``degrees of freedom'' on the LASSO (H. Zou, et al. $2007$), is classified as Mixed Topic paper.
We can simply guess model selection has lots of applications in other topics, so it might cite or have been cited by many papers in other communities.
Actually, out of $11$ citation relationships it has with other papers, $6$ of them came from the relationships with papers from Mixed Topics.

Non-zero components of $\widehat{S}$ capture the citation relationships among papers that are not attributable to the common topics.
The selected model has $51$ sparse edges, and all of them are positive edges.
In Table \ref{tab:table2}, we provide $10$ pairs of papers that have the largest estimated $\widehat{S}_{ij}$.
All the $10$ edges come from the pairs of papers from different topics.
For instance, the first pair of papers comes from the Functional Analysis topic and Variable Selection topic.
The paper from Functional Analysis topic cites the paper from Variable Selection for borrowing a mathematical representation to build a theorem.
Though it might appear to be a crucial step for building a theorem in their paper, we cannot say that two papers are closely related in terms of topic.
The second pair of papers comes from the Mixed Topics community and the Variable Selection community.
This case is interesting since both papers study about variable selection problem, but they are classified in different communities and connected via an ad-hoc link.
Specifically, the authors in the paper from Mixed Topics community study the variable selection problem under the Non-parametric Bayesian framework, and compare their method with the ``Adaptive Lasso'' that is proposed in the paper from Variable Selection topic.



\begin{table}[htbp]
\centering
\begin{tabular}{clcl}\hline
Pair & Community & Title  & \\
\hline
1    & FuncAn    & Properties of principal component methods for functional and longitudinal data analysis       & \\
     & VarSel    & Nonconcave penalized likelihood with a diverging number of parameters                        & \\
2    & VarSel      & The adaptive lasso and its oracle properties                              & \\
     & Mixed    & Nonparametric Bayes conditional distribution modeling with variable selection     & \\
3    & DimRed    & Contour projected dimension reduction                                                      & \\
     & VarSel    & Factor profiled sure independence screening & \\
4    & VarSel    & Factor profiled sure independence screening & \\
     & DimRed    & Sliced regression for dimension reduction  & \\
5    & CovEst    & Two sample tests for high-dimensional covariance matrices & \\
     & VarSel      & The sparsity and bias of the {LASSO} selection in high-dimensional linear regression         & \\
6    & MulT    & Innovated higher criticism for detecting sparse signals in correlated noise      & \\
     & CovEst     & Regularized estimation of large covariance matrices   &  \\
7    & DimRed    & A constructive approach to the estimation of dimension reduction directions       & \\
     & VarSel     & Factor profiled sure independence screening   & \\
8    & VarSel    & A majorization-minimization approach to variable selection using spike and slab priors & \\
     & Mixed      & Empirical {B}ayes selection of wavelet thresholds               & \\
9    & Mixed    & Nonparametric inferences for additive models          & \\
     & VarSel    & Nonparametric independence screening in sparse ultra-high-dimensional additive model  & \\
10   & VarSel    & Sure independence screening in generalized linear models with {NP}-dimensionality & \\
     & Mixed     & Maximum likelihood estimation in semi-parametric regression models with censored data & \\
\hline \\
\end{tabular}

\begin{tabular}{ccccc}
\hline
Pair $1$                                                                                 & Pair $2$                                                                       & Pair $3$                                                                             & Pair $4$                                                                            & Pair $5$                                                                              \\ \hline
\begin{tabular}[c]{@{}c@{}}P. Hall, et al. $2006$\\ J. Fan, et al. $2004$\end{tabular}     & \begin{tabular}[c]{@{}c@{}}H. Zou. $2006$\\ Y. Chung, et al. $2009$\end{tabular} & \begin{tabular}[c]{@{}c@{}}R. Luo, et al. $2009$\\ H. Wang. $2012$\end{tabular}        & \begin{tabular}[c]{@{}c@{}}H. Wang. $2012$\\ H. Wang, et al. $2012$\end{tabular}      & \begin{tabular}[c]{@{}c@{}}J. Li, et al. $2012$\\ CH. Zhang, et al. $2008$\end{tabular} \\ \hline
Pair $6$                                                                                 & Pair $7$                                                                       & Pair $8$                                                                             & Pair $9$                                                                            & Pair $10$                                                                             \\ \hline
\begin{tabular}[c]{@{}c@{}}P. Hall, et al. $2010$\\ PJ. Bickel, et al. $2008$\end{tabular} & \begin{tabular}[c]{@{}c@{}}Y. Xia. $2007$\\ H. Wang. $2012$\end{tabular}         & \begin{tabular}[c]{@{}c@{}}TJ. Yen. $2011$\\ IM. Johnstone, et al. $2005$\end{tabular} & \begin{tabular}[c]{@{}c@{}}J. Fan, et al. $2005$\\ J. Fan, et al. $2011$\end{tabular} & \begin{tabular}[c]{@{}c@{}}J. Fan, et al. $2010$\\ D. Zeng, et al. $2007$\end{tabular}
\\ \hline
\end{tabular}
\caption{Top $10$ edges corresponding with the pairs of papers from different communities.
Authors and years of publication for the papers in each pair are also presented.
In the first pair, a $\emph{FuncAn}$ paper cites a $\emph{VarSel}$ paper for borrowing a mathematical representation to build a theorem.
But they are not related in terms of {\it topic}.}
\label{tab:table2}
\end{table}

\section{Discussion}
\label{sec:Con}

We propose a new model that combines the latent factors and a sparse graphical structure.
We consider the regularized likelihood by means of the $L_1$ norm and the nuclear norm penalties.
The computation of the regularized estimator is facilitated by developing an algorithm that is based on the alternating direction method of multiplier (ADMM), which optimizes a non-smooth however convex objective function.
The proposed method is applied to a citation network of statisticians, and the estimated model renders some meaningful interpretations.
We believe that our analysis on statistician's citation network sheds some new light on the interpretation of the data set.

There are still several questions remaining to be answered.
First of all, it remains unclear on how to choose the tuning parameters.
Classical methods for choosing tuning parameters such as BIC or AIC tend to choose the most parsimonious models, which lead to an underestimation in our case.
We also do not have systematic ways to do cross-validation in our network data.
Not only because it is computationally expensive, but also because if we partition the network data, we can loose fair amount of information on dependent structures among the edges.
This problem is also closely related to determining the number of communities in network.
In lieu of using BIC or AIC, our analysis is heavily relying on a heuristic approach when choosing the tuning parameter, and during this procedure, we use the scree-plot to determine the number of communities in a network.
Screeplot approach works well in general situation, but it does not necessarily always guarantee the correct estimate of number of communities.
We need a more systematic way to choose the parameters.
And it will be nice to derive some theoretical guarantees for the methods.

Secondly, we only consider an undirected graph, which is somewhat inconsistent with a real citation network, which is directional.
Since, in our research, we were interested in separating the low rank structure of edges and ad-hoc links in network, we did not take into account the directions of edges in our model.
However, it would be interesting to study a similar problem on a directed graph.
This is a future work.

Last but not least, when we assign the memberships of each nodes, we adopt the $k$-means clustering algorithm.
However, we notice that the $k$-means algorithm tends to assign nodes conservatively to each communities.
For example, in Fig. \ref{fig:figure4} (left), we can see that a bunch of Multiple Testing papers are assigned as Mixed cluster, and in Fig. \ref{fig:figure4} (right), many papers that should have been classified to three communities other than the Mixed topic, have been assigned into the Mixed topic community.
It would be interesting to experiment on some other clustering methods that accommodate overlapping memberships.

\section{Appendix}

In this Section, first we briefly introduce several notations, including a notion on the decomposability of regularizer, and a useful lemma that is proved in the work \cite{agarwal2012noisy} (Section\ref{App:Pre}).
Then, we present Lemma \ref{le:le2} and its proof (Section \ref{App:Lemma2}).
Finally, we present the proof of our Theorem \ref{Th:th1} (Section \ref{App:Theorem1}).

\subsection{Preliminary} \label{App:Pre}
Throughout the proof, we adopt the convenient short-hand notation on projection of matrix $P$ on subspace $M$ as $P_{M}$.
We use $\langle A,B \rangle$ to denote the trace inner product of two matrices $A$ and $B$ \big(i.e.,$\langle A,B \rangle=\mbox{tr}\big(A^T B\big)$\big).
We use $\|A\|_{\infty}$ to denote the maximum absolute entry of matrix $A$, and use $\|B\|_{op}$ to denote the largest singular value of matrix $B$.
And we will use the notion of decomposability of $L_1$ norm with respect to a pair of subspace $(M,M^{\perp})$.
Given an arbitrary subset $S \subseteq \{1,2,\dots,n\} \times \{1,2,\dots,n\}$ of matrix indices, $M$ is defined as follows:
\[
    M(S):=\{ U \in \mathbb{R}^{n \times n} | U_{ij}=0, \forall (i,j) \in S \}
\]
and $M^{\perp}(S):=(M(S))^{\perp}$.
With this in mind, we recall the formal definition of the decomposability of $L_1$ norm as follows:
\begin{definition}
    Given a subspace $M \subset \mathbb{R}^{n \times n}$ and its orthogonal complement $M^{\perp}$, an elementwise $L_1$ norm
    is decomposable with respect to $(M, M^{\perp})$ if
    \[
        \| A + B \|_{1} = \|A\|_{1} + \|B\|_{1},
        \forall A \in M \mbox{ and } B \in M^{\perp}.
    \]
\end{definition}
The notion of decomposability is used to penalize the perturbation from the model subspace $M$, and to obtain the tightest bound the $L_1$ norm can achieve.
We will also use two results in our proof, which are presented and proved in the work \cite{agarwal2012noisy}.
For the convenience of readers, we present them here:
\begin{lemma} \label{eq:Lmag}
(\textbf{Agarwal, et al} \cite{agarwal2012noisy})
For any $k=1,2,\dots,n$, there is a decomposition $\widehat{\Delta}^{L} =\widehat{\Delta}^{L}_{A} + \widehat{\Delta}^{L}_{B}$ such that:
\begin{enumerate}
 \item The decomposition satisfies
    \begin{equation}\label{eq:Lm1}
        \mbox{rank}\big(\widehat{\Delta}^{L}_{A}\big) \leq 2k, \quad and \quad \big(\widehat{\Delta}^{L}_{A}\big)^{T}\widehat{\Delta}^{L}_{B}
        = \big(\widehat{\Delta}^{L}_{B}\big)^{T}\widehat{\Delta}^{L}_{A} =0
    \end{equation}
 \item The difference $\mathbb{Q}\big(L^{*},S^{*}\big) -  \mathbb{Q}\big(\widehat{\Delta}^L + L^{*},\widehat{\Delta}^S + S^{*}\big)$ is upper-bounded by
    \begin{equation} 
        \mathbb{Q}\big(\widehat{\Delta}^L_{A},\widehat{\Delta}^S_{M}\big) - \mathbb{Q}\big(\widehat{\Delta}^L_{B},\widehat{\Delta}^S_{M^\perp}\big)
        +2 \sum_{j=k+1}^{n} \sigma_{j}\big(L^*\big) + 2\frac{\gamma}{\delta}\left\|S^*_{M^\perp}\right\|_{1},
    \end{equation}
\end{enumerate}
\end{lemma}
where the notation $\mathbb{Q}(L,S)$ is defined as the weighted combination of the two regularizers for any pair of positive tuning parameters $(\gamma,\delta)$:
\[
\mathbb{Q}\left(L,S\right)   := \left\|L\right\|_{*} + \frac{\gamma}{\delta}\left\|S\right\|_1.
\]

\subsection{Lemma \ref{le:le2}} \label{App:Lemma2}
\begin{lemma} \label{le:le2}
If a pair of regularization parameters $(\delta,\gamma)$ satisfies condition (\ref{eq:49}), then for $\mathbb{Q}\left(\widehat{\Delta}^L_{B},\widehat{\Delta}^S_{M^\perp}\right)$, we have
\[
    \mathbb{Q}\left(\widehat{\Delta}^L_{B},\widehat{\Delta}^S_{M^\perp}\right) \leq
    \left\|\hat{\Delta}^{\alpha}\mathbbm{1}\mathbbm{1}^T\right\|_{F} +
    3\mathbb{Q}\left(\widehat{\Delta}^{L}_{A},\widehat{\Delta}^{S}_{M}\right)+4 \sum_{j=k+1}^{n} \sigma_{j}\big(L^{*}\big) + 4\frac{\gamma}{\delta}
    \left\|S^{*}_{M^\perp}\right\|_{1}.
\]
\end{lemma}

\begin{proof}
Through the application of basic inequality by using optimality of $\widehat{\Theta}$ and feasibility of $\Theta^{*}$ to convex program (\ref{eq:Modified}), we have
\begin{equation}
    h\big(\widehat{\Theta}\big) - h\big(\Theta^{*}\big)
    \leq \delta \mathbb{Q}\big(L^{*},S^{*}\big) - \delta \mathbb{Q}\big(\widehat{\Delta}^L + L^{*},\widehat{\Delta}^S + S^{*}\big).
\end{equation}
By using convexity of $h(\Theta)$, we can write
\begin{align}
    h\big(\widehat{\Theta}\big)-h\big(\Theta^{*}\big) &\geq
    \big\langle\, \nabla_{\Theta}h(\Theta^{*}),\widehat{\Theta}-\Theta^{*}\big\rangle\ \nonumber \\
    &= - \big\langle\,\frac{1}{n} (X-P^{*}), \widehat{\Delta}^{\alpha}\mathbbm{1}\mathbbm{1}^T + \widehat{\Delta}^{L} + \widehat{\Delta}^{S} \big\rangle\ \nonumber \\
    &\geq -\frac{1}{n} \|X - P^*\|_{op}\bigg(\big\|\widehat{\Delta}^{\alpha}\mathbbm{1}\mathbbm{1}^T\big\|_\ast + \big\|\widehat{\Delta}^{L}\big\|_\ast\bigg) +  \frac{1}{n} \|X - P^*\|_{\infty}\big\|\widehat{\Delta}^{S}\big\|_{1} \nonumber \\
    &\geq -\frac{\delta}{2}\bigg(\big\|\widehat{\Delta}^{\alpha}\mathbbm{1}\mathbbm{1}^T\|_F + \|\widehat{\Delta}^{L}_{A}\big\|_\ast + \big\|\widehat{\Delta}^{L}_{B}\big\|_\ast\bigg) - \frac{\gamma}{2} \bigg(\big\|\widehat{\Delta}^{S}_{M}\big\|_{1} +
    \big\|\widehat{\Delta}^{S}_{M^\perp}\big\|_{1}\bigg). \label{eq:27}
\end{align}

An application of Agarwal et al \cite{agarwal2012noisy}'s second element of lemma \ref{eq:Lmag},
we can get an upper bound of difference $\mathbb{Q}\big(L^*,S^*\big)- \mathbb{Q}\big(\widehat{\Delta}^L + L^{*},\widehat{\Delta}^S + S^{*}\big)$ as follows:

\begin{equation} \label{eq:28}
    \mathbb{Q}\big(\widehat{\Delta}^L_{A},\widehat{\Delta}^S_{M}\big) - \mathbb{Q}\big(\widehat{\Delta}^L_{B},\widehat{\Delta}^S_{M^\perp}\big)
    +2 \sum_{j=k+1}^{n} \sigma_{j}\big(L^*\big) + 2\frac{\gamma}{\delta}\big\|S^*_{M^\perp}\big\|_{1}.
\end{equation}

By combining relations (\ref{eq:27}) and (\ref{eq:28}), we can get the upper bound of $\mathbb{Q}\big(\widehat{\Delta}^L_{B},\widehat{\Delta}^S_{M^\perp}\big)$ :
\begin{align*}
    \mathbb{Q}\big(\widehat{\Delta}^L_{B},\widehat{\Delta}^S_{M^\perp}\big) \leq
    \big\|\widehat{\Delta}^{\alpha}\mathbbm{1}\mathbbm{1}^T\big\|_{F} +
    3\mathbb{Q}\big(\widehat{\Delta}^{L}_{A},\widehat{\Delta}^{S}_{M}\big)+4 \sum_{j=k+1}^{n} \sigma_{j}\big(L^{*}\big) + 4\frac{\gamma}{\delta}
    \big\|S^{*}_{M^\perp}\big\|_{1}.
\end{align*}
\end{proof}

\subsection{Proof of Theorem \ref{Th:th1}} \label{App:Theorem1}
\begin{proof}
Since $\widehat{\Theta}$ and $\Theta^{*}$ are optimal minimizer and feasible solution respectively for the convex program (\ref{eq:Modified}), we have
\begin{equation}\label{eq:2}
    h\big(\widehat{\Theta}\big) + \delta\big\|\widehat{L}\big\|_{*} + \gamma\big\|\widehat{S}\big\|_{1} \leq
     h\big(\Theta^{*}\big) + \delta\big\|{L^{*}}\big\|_{*} + \gamma\big\|{S}^{*}\big\|_{1}.
\end{equation}

Through the assumption of strong convexity on $h(\Theta)$, and by the Taylor expansion, we can get a following lower bound on the term $h\big(\widehat{\Theta}\big)-h\big(\Theta^{*}\big)$ :
\[
h\big(\widehat{\Theta}\big)-h\big(\Theta^{*}\big) \geq
\big\langle\, \nabla_{\Theta}h\big(\Theta^{*}\big),\widehat{\Theta}-\Theta^{*}\big\rangle\ + \frac{\tau}{2}\big\|\widehat{\Delta}^{\Theta}\big\|_{F}^{2}.
\]

By rearranging the term in (\ref{eq:2}) and plugging in above inequality relation, we get:

\begin{equation}\label{eq:3}
    \frac{\tau}{2}\big\|\widehat{\Delta}^{\Theta}\big\|_{F}^{2} \leq
    -\big\langle\,\nabla_{\Theta}h(\Theta^{*}),\widehat{\Theta}-\Theta^{*}\big\rangle\
    + \delta \big\|L^{*}\big\|_\ast + \gamma \big\|S^{*}\big\|_1
    - \delta \big\|\widehat{L}\big\|_\ast - \gamma \big\|\widehat{S}\big\|_1.
\end{equation}

Through the definition of $\mathbb{Q}$, we can rewrite (\ref{eq:3}) as follows:
\begin{equation}\label{eq:4}
    \frac{\tau}{2}\big\|\widehat{\Delta}^{\Theta}\big\|_{F}^{2} \leq
    -\big\langle\,\nabla_{\Theta}h\big(\Theta^{*}\big),\widehat{\Theta}-\Theta^{*}\big\rangle\
    + \delta \mathbb{Q}\big(L^{*},S^{*}\big) - \delta \mathbb{Q}\big(\widehat{\Delta}^L + L^{*},\widehat{\Delta}^S + S^{*}\big).
\end{equation}

According to Agarwal et al \cite{agarwal2012noisy}'s second element of lemma \ref{eq:Lmag},
the difference $\mathbb{Q}\big(L^*,S^*\big)- \mathbb{Q}\big(\widehat{\Delta}^L + L^{*},\widehat{\Delta}^S + S^{*}\big)$ is upper-bounded by

\begin{equation}\label{eq:5}
    \mathbb{Q}\big(\widehat{\Delta}^L_{A},\widehat{\Delta}^S_{M}\big) - \mathbb{Q}\big(\widehat{\Delta}^L_{B},\widehat{\Delta}^S_{M^\perp}\big)
    +2 \sum_{j=k+1}^{n} \sigma_{j}\big(L^*\big) + 2\frac{\gamma}{\delta}\big\|S^*_{M^\perp}\big\|_{1}.
\end{equation}

First, we want to control upper bound of the term $-\big\langle\,\nabla_{\Theta}h(\Theta^{*}),\widehat{\Theta}-\Theta^{*} \big\rangle\,$
in (\ref{eq:4}).
\begin{align}
-\big\langle\,\nabla_{\Theta}h(\Theta^{*}),\widehat{\Theta}-\Theta^{*} \big\rangle\  \nonumber
&=\big\langle\, \frac{1}{n} (X - P^*), \widehat\Delta^{\alpha\mathbbm{1}\mathbbm{1}^T} + \widehat\Delta^{L} + \widehat\Delta^{S} \big\rangle\ \\ \nonumber
&\leq \frac{1}{n} \|X - P^*\|_{op}\bigg(\big\|\widehat{\Delta}^{\alpha}\mathbbm{1}\mathbbm{1}^T\big\|_\ast + \big\|\widehat{\Delta}^{L}\big\|_\ast\bigg) +  \frac{1}{n} \|X - P^*\|_{\infty}\big\|\widehat{\Delta}^{S}\big\|_{1} \\ \nonumber
&\leq \frac{1}{n} \|X - P^*\|_{op}\bigg(\left\|\widehat{\Delta}^{\alpha}\mathbbm{1}\mathbbm{1}^T\right\|_{F} + \left\|\widehat{\Delta}^{L}_{A}\right\|_\ast + \left\|\widehat{\Delta}^{L}_{B}\right\|_\ast\bigg) +
\frac{1}{n} \left\|X - P^*\right\|_{\infty}\bigg(\left\|\widehat{\Delta}^{S}_{M}\right\|_{1} +
\left\|\widehat{\Delta}^{S}_{M^{\perp}}\right\|_{1}\bigg) \\
&\leq \frac{\delta}{2}\bigg(\big\|\hat{\Delta}^{\alpha}\mathbbm{1}\mathbbm{1}^T\big\|_{F} + \big\|\widehat{\Delta}^{L}_{A}\big\|_\ast + \big\|\widehat{\Delta}^{L}_{B}\big\|_\ast\bigg) + \frac{\gamma}{2}\bigg(\left\|\widehat{\Delta}^{S}_{M}\right\|_{1} +
\left\|\widehat{\Delta}^{S}_{M^{\perp}}\right\|_{1}\bigg).  \label{eq:10}
\end{align}

Combining the inequalities (\ref{eq:5}) and (\ref{eq:10}), we can obtain the upper bound of RHS in (\ref{eq:4}) as follows:

\begin{equation}
    \frac{\tau}{2}\big\|\widehat{\Delta}^{\Theta}\big\|_{F}^{2} \leq
    \frac{\delta}{2}\big\|\widehat{\Delta}^{\alpha}\mathbbm{1}\mathbbm{1}^T\big\|_{F} +
    \frac{3\delta}{2}\mathbb{Q}\big(\widehat{\Delta}^L_{A},\widehat{\Delta}^S_{M}\big)
    +2 \delta \sum_{j=k+1}^{n} \sigma_{j}\big(L^*\big) + 2\gamma \big\|S^*_{M^\perp}\big\|_{1}.
    \label{eq:11}
\end{equation}

Second, we wish to control the lower bound of the term  $\frac{\tau}{2}\big\|\widehat{\Delta}^{\Theta}\big\|_{F}^{2}$ with respect to $\widehat{\Delta}^{\alpha},\widehat{\Delta}^{L},\widehat{\Delta}^{S}$.
\begin{align}
\big\|\widehat{\Delta}^{\Theta}\big\|_{F}^{2}  \nonumber
&= \big\|\widehat{\Theta}-\Theta^{*} \big\|_{F}^{2} \\ \nonumber
&= \big\|\widehat\Delta^{\alpha\mathbbm{1}\mathbbm{1}^T} + \widehat\Delta^{L} + \widehat\Delta^{S}\big\|_{F}^{2} \\ \nonumber
&= \big\|\widehat{\Delta}^{\alpha}\mathbbm{1}\mathbbm{1}^T\big\|_{F}^{2} + \big\|\widehat{\Delta}^{L}+\widehat{\Delta}^{S}\big\|_{F}^{2} +
2 \big\langle\, \widehat\Delta^{L} + \widehat\Delta^{S} ,\widehat{\Delta}^{\alpha}\mathbbm{1}\mathbbm{1}^T \big\rangle\ \\
&=  \big\|\widehat{\Delta}^{\alpha}\mathbbm{1}\mathbbm{1}^T\big\|_{F}^{2} + \big\|\widehat{\Delta}^{L}\big\|_{F}^{2} + \big\|\widehat{\Delta}^{S}\big\|_{F}^{2} +
2 \big\langle\, \widehat\Delta^{L} + \widehat\Delta^{S} ,\widehat{\Delta}^{\alpha}\mathbbm{1}\mathbbm{1}^T \big\rangle\ +
2 \big\langle\, \widehat\Delta^{L}, \widehat\Delta^{S} \big\rangle.\label{eq:12}
\end{align}

We want to get the further lower bound on trace inner product terms,
$ \big\langle\, \widehat\Delta^{L} + \widehat\Delta^{S} ,\widehat{\Delta}^{\alpha}\mathbbm{1}\mathbbm{1}^T \big\rangle\ $, $ \big\langle\, \widehat\Delta^{L}, \widehat\Delta^{S} \big\rangle\ $. To control the first trace inner product term, we use the relation $\widehat{\Delta}^{L}\mathbbm{1}=0$, apply the definition of dual norm on inner product term, apply triangular inequality on $\widehat{\Delta}^\alpha$, and lastly we apply the constraint imposed on $|\alpha|$ stated in Assumption \ref{Ass:2}.

\begin{align}
    \big| \big\langle\, \widehat\Delta^{L} + \widehat\Delta^{S} ,\widehat{\Delta}^{\alpha}\mathbbm{1}\mathbbm{1}^T\big\rangle\ \big| \nonumber
    &= \big| \big\langle\,\widehat\Delta^{S},\widehat{\Delta}^{\alpha}\mathbbm{1}\mathbbm{1}^T\big\rangle\ \big| \\ \nonumber
    &\leq \big\|\widehat{\Delta}^{\alpha}\mathbbm{1}\mathbbm{1}^T \big\|_{\infty} \big\|\widehat{\Delta}^{S}\big\|_1 \\ \nonumber
    &\leq \bigg(\big|\widehat{\alpha}\big|+\big|\alpha^{*}\big|\bigg)\big\|\widehat{\Delta}^{S}\big\|_1\\
    &\leq 2C\kappa\big\|\widehat{\Delta}^{S}\big\|_1.  \label{eq:13}
\end{align}

To control the term $ \big\langle\, \widehat\Delta^{L}, \widehat\Delta^{S} \big\rangle\ $, we first apply the definition of dual norm on trace inner product term, then apply triangular inequality on $\widehat{\Delta}^{L}$ and spikiness condition.

\begin{align}
    \big| \big\langle\, \widehat\Delta^{L}, \widehat\Delta^{S} \big\rangle\ \big| \nonumber
    &\leq \big\|\widehat{\Delta}^{L}\big\|_{\infty} \big\|\widehat{\Delta}^{S}\big\|_1\\ \nonumber
    &\leq \bigg(\big\|\widehat{L}\big\|_{\infty} + \big\|L^{*}\big\|_{\infty}\bigg)
    \big\|\widehat{\Delta}^{S} \big\|_{1}\\
    &\leq \bigg(\frac{2\kappa}{n}\bigg)\big\|\widehat{\Delta}^{S} \big\|_{1}. \label{eq:14}
\end{align}

We can combine the inequality (\ref{eq:12}), (\ref{eq:13}) and (\ref{eq:14}). Then applying the assumption on regularization parameter $\gamma$, and the fact $\big\|\widehat{\Delta}^{L}\big\|_{\ast}\geq0$ sequentially, we can get,

\begin{align}
    \frac{\tau}{2}\big\|\widehat{\Delta}^{\Theta}\big\|_{F}^{2}  \nonumber
    &\geq \frac{\tau}{2} \big\|\widehat{\Delta}^{\alpha}\mathbbm{1}\mathbbm{1}^T\big\|_{F}^{2} + \frac{\tau}{2} \big\|\widehat{\Delta}^{L}\big\|_{F}^{2} + \frac{\tau}{2} \big\|\widehat{\Delta}^{S}\big\|_{F}^{2}
    - \kappa\tau\bigg(\frac{Cn+1}{n}\bigg)\big\|\widehat{\Delta}^{S}\big\|_{1} \\ \nonumber
    &\geq \frac{\tau}{2} \big\|\widehat{\Delta}^{\alpha}\mathbbm{1}\mathbbm{1}^T\big\|_{F}^{2} + \frac{\tau}{2} \big\|\widehat{\Delta}^{L}\big\|_{F}^{2} + \frac{\tau}{2} \big\|\widehat{\Delta}^{S}\big\|_{F}^{2} - \frac{\gamma}{2}\big\|\widehat{\Delta}^{S}\big\|_{1}\\
    &\geq \frac{\tau}{2} \big\|\widehat{\Delta}^{\alpha}\mathbbm{1}\mathbbm{1}^T\big\|_{F}^{2} + \frac{\tau}{2} \big\|\widehat{\Delta}^{L}\big\|_{F}^{2} + \frac{\tau}{2} \big\|\widehat{\Delta}^{S}\big\|_{F}^{2} - \frac{\delta}{2} \mathbb{Q}\big(\widehat{\Delta}^L,\widehat{\Delta}^S\big). \label{eq:26}
\end{align}

By combining the relations (\ref{eq:11}) and (\ref{eq:26}), applying triangular inequality, $\mathbb{Q}\big(\widehat{\Delta}^L,\widehat{\Delta}^S\big)\leq \mathbb{Q}\big(\widehat{\Delta}^L_{A},\widehat{\Delta}^S_{M}\big) + \mathbb{Q}\big(\widehat{\Delta}^L_{B},\widehat{\Delta}^S_{M^{\perp}}\big)$, and rearranging the term, we can get following inequality,
\[
    \frac{\tau}{2} \big\|\widehat{\Delta}^{\alpha}\mathbbm{1}\mathbbm{1}^T\big\|_{F}^{2} + \frac{\tau}{2} \big\|\widehat{\Delta}^{L}\big\|_{F}^{2} + \frac{\tau}{2} \big\|\widehat{\Delta}^{S}\big\|_{F}^{2} \\
    \leq \frac{\delta}{2} \big\|\widehat{\Delta}^{\alpha}\mathbbm{1}\mathbbm{1}^T\big\|_{F} +
    2\mathbb{Q}\big(\widehat{\Delta}^L_{A},\widehat{\Delta}^S_{M}\big) +  \frac{\delta}{2} \mathbb{Q}\big(\widehat{\Delta}^L_{B},\widehat{\Delta}^S_{M^\perp}\big) +
    2 \delta \sum_{j=k+1}^{n} \sigma_{j}\big(L^{*}\big) + 2\gamma
    \big\|S^{*}_{M^\perp}\big\|_{1}.
\]

Further, by plugging in Lemma 1 to get an upper bound on  $\mathbb{Q}(\widehat{\Delta}^L_{B},\widehat{\Delta}^S_{M^\perp})$, we can rewrite the above inequality as follows:
\begin{align}
    \frac{\tau}{2} \big\|\widehat{\Delta}^{\alpha}\mathbbm{1}\mathbbm{1}^T\big\|_{F}^{2} + \frac{\tau}{2} \big\|\widehat{\Delta}^{L}\big\|_{F}^{2} + \frac{\tau}{2} \big\|\widehat{\Delta}^{S}\big\|_{F}^{2} -
    \frac{\delta}{2} \big\|\widehat{\Delta}^{\alpha}\mathbbm{1}\mathbbm{1}^T\big\|_{F}
    \leq
    \frac{7\delta}{2}\mathbb{Q}\big(\widehat{\Delta}^L_{A},\widehat{\Delta}^S_{M}\big) + 4 \delta \sum_{j=k+1}^{n} \sigma_{j}\big(L^*\big) + 4\gamma \big\|S^*_{M^\perp}\big\|_{1}.  \label{eq:42}
\end{align}
Noting that $\widehat{\Delta}^{L}_{A}$ has rank at most 2$k$ and that $\widehat{\Delta}^{S}_{M}$ lies in the model space $M$, we find that
\begin{align}
    \nonumber
    \delta\mathbb{Q}\big(\widehat{\Delta}^L_{A},\widehat{\Delta}^S_{M}\big)
    \leq \sqrt{2k}\delta\big\|\widehat{\Delta}^L_{A}\big\|_{F} + \Psi(M)\gamma\big\|\widehat{\Delta}^S_M\big\|_{F}\\
    \leq \sqrt{2k}\delta\big\|\widehat{\Delta}^L\big\|_{F} + \Psi(M)\gamma\big\|\widehat{\Delta}^S\big\|_{F}.  \label{eq:43}
\end{align}
Here $\Psi(M)$ measures the compatibility between Frobenius norm and component-wise $L_{1}$ regularizer, where $M$ is an arbitrary subset of matrix indices of cardinality at most s.
\[
    \Psi(M):=\sup\limits_{U\in M,U\neq0}\frac{\|U\|_{1}}{\|U\|_{F}}.
\]
Using Cauchy-Schwarz inequality, we can easily check the quantity $\Psi(M)$ is bounded by at most $\sqrt{s}$. Plugging in the relation ($\ref{eq:43}$) into ($\ref{eq:42}$) and rearranging the term relevant with  $e^{2}\big(\hat{\alpha}\mathbbm{1}\mathbbm{1}^T,\widehat{L},\widehat{S}\big)$ yield the claim.
\end{proof}

\bibliographystyle{plain}
\bibliography{main}

}\end{titlingpage}
\end{document}